\documentclass{article}
\usepackage{ijcai26}

\usepackage{times}
\usepackage{soul}
\usepackage{url}
\usepackage[hidelinks]{hyperref}
\usepackage[utf8]{inputenc}
\usepackage[small]{caption}
\usepackage{graphicx}
\usepackage{amsmath}
\usepackage{amsthm}
\usepackage{booktabs}
\usepackage{algorithm}
\usepackage{algorithmic}
\usepackage[switch]{lineno}

\usepackage{amsmath}
\usepackage{amssymb} 
\usepackage{booktabs}
\usepackage{cuted}

\newtheorem{proposition}{Proposition}
\newtheorem{corollary}{Corollary}[proposition]

\newtheorem{conjecture}{Conjecture}

\title{A Fair Objective for Human-Empowerment-Preserving AI:\\ Desiderata, Design, and Likely Behavioral Consequences}
\makeatletter
\let\mytitle\@title
\makeatother

\author{
    \textsuperscript{\rm 1},
    \textsuperscript{\rm 2}
}

\iftrue
\author{
Jobst Heitzig$^1$
\and
Ram Potham$^2$\\
\affiliations
$^1$Potsdam Institute for Climate Impact Research, Potsdam, Germany\\
$^2$Independent\\
\emails
heitzig@pik-potsdam.de,
ram.potham@gmail.com
}
\fi

\DeclareMathOperator*{\E}{\mathbb{E}}

\DeclareMathOperator*{\argmax}{arg\,max}

\def\A{\mathcal{A}}  
\def\G{\mathcal{G}}  
\def\H{\mathcal{H}}  
\def\S{\mathcal{S}}  

\def\ld{\log_2}

\def\Pr{\mathrm{Pr}} 

\def\entropy{\mathbb{H}} 
\def\MI{\mathbb{I}} 
\def\DKL{D_{KL}} 

\def\klyubinE{\mathfrak{E}}

\def\trans{\Pr}  
\def\GAC#1#2{C_{#1}^{#2}}  
\def\IP#1{I_{#1}}  
\def\PP{P}  
\def\TP{T}  
\def\LP{L}  

\def\citet#1{\citeauthor{#1} [\citeyear{#1}]}

\usepackage{todonotes}

\usepackage{xcolor}
\usepackage{url}

\begin{document}

\maketitle

\begin{abstract}
This paper explores the idea of promoting well-being and safety in human-AI interactions by forcing AI agents explicitly to empower humans and to manage the power balance between humans and AI agents in a desirable way.
Using a principled, partially axiomatic approach based on desirable properties, we design a parametrizable and decomposable objective function for AI systems that represents an inequality- and risk-averse long-term aggregate of human power. 
It can take into account models of human bounded rationality and social norms, and crucially, considers a wide variety of possible human goals.
We prove how certain desiderata enforce particular functional forms and restrict parameter ranges.
We exemplify the consequences of softly maximizing this metric in several paradigmatic situations and describe what instrumental sub-goals it will likely imply.
\end{abstract}


\section{Introduction}

A duty to empower others, especially those with little power, can be defended on consequentialist \cite{sen2014development}, deontological \cite{hill2002human}, and virtue ethical \cite{nussbaum2019aristotelian} grounds.
At the same time, it has been argued that automated AI decision making can gradually {\em dis}empower humans
\cite{vallor2015moral,ferdman2025ai,kulveit2025position}.

Accordingly, some papers explored tasking AI agents explicitly with the empowerment of (individual) human users, using the channel capacity $\klyubinE$ between human actions and environmental states, called `empowerment' in \citet{klyubin2005empowerment}, as the key metric \cite{salge2017empowerment}.
However, that purely information-theoretic metric lacks a clear interpretation in terms of the capability to attain {\em goals,} which is arguably a key aspect of (em)power(ment) \cite{lukes2026power}.

\paragraph{Related work}
Several threads in AI safety address the question of what objective an AI agent should pursue when interacting with humans.
Assistance games \cite{hadfield2016cooperative} assume the agent is uncertain about the human's utility function and learns it through interaction, but require that a true utility function exists and is identifiable.
Impact measures aim to keep agents conservative by penalizing deviation from a baseline; attainable utility preservation \cite{turner2020conservative} and reachability analysis \cite{krakovna2018measuring} both quantify what an agent's actions foreclose, but do so from the {\em agent's} perspective rather than measuring {\em human} capability.
The `empowerment' literature \cite{klyubin2005empowerment,salge2017empowerment} proposes maximizing the information-theoretic channel capacity between a human's actions and resulting states, but does not account for bounded rationality, goal attainment, or aggregation across several humans.
Corrigibility \cite{potham2025corrigibility} addresses whether agents permit human correction, typically as a standalone desideratum rather than as an emergent consequence of a unified objective.
Our approach connects these threads: like assistance games it is grounded in human agency, like impact measures it cares about preserving options, like `empowerment' it uses an intrinsic structural metric, and like social choice theory it aggregates across individuals using axiomatic methods. The key difference is that we derive a single, axiomatically motivated objective from which safety properties such as corrigibility, norm-following, and commitment-making emerge rather than being imposed separately.

\paragraph{Contribution} Concretely, this paper extends the theoretical foundations of the empowerment-based approach to AI in three ways.
First, inspired by the ``capability approach'' to welfare \cite{sen2014development}, 
we develop an {\em alternative metric} of human power that is directly based on the capability to attain a wide range of possible goals. 
As our metric is based on some given model of human behavior, it can explicitly and transparently incorporate humans' knowledge about the actions of the AI agent, their expectations about others' behavior, e.g.\ due to social norms, and their bounded rationality. 

Second, using an approach guided by desiderata similar to decision theory, social choice theory, and welfare theory, we develop an objective function for an AI agent interacting with {\em populations} of humans, based on a fair aggregate of power across humans and along time.

Finally, we analyze the resulting system behavior in a set of stylized situations and argue that the empowerment-based approach gives the AI agent desirable incentives such as communicating well \cite{reddy2022first}, following orders, being corrigible \cite{potham2025corrigibility}, avoiding irreversible changes in the environment, protecting humans and itself from harm and disempowerment, acting ``appropriately'' by following relevant social norms \cite{leibo2024theory}, and allocating resources fairly and sustainably.

Though based on {\em possible} human goals, our approach avoids relying on the assumption that the AI system knows or is able to learn individual humans' {\em actual, current} goals, as it has been argued that human goals are changing and non-identifiable \cite{cao2021identifiability,banerjee2011poor} and their prediction unavoidably uncertain \cite{baker2011bayesian}.\footnote{%
    Some critics of a goal- or preference-based approach argue for a {\em values}-based approach instead \cite{lowe2025fullstack}, which however seems to require even stronger assumptions on the AI agent's semantic world understanding. A deep semantic understanding is also required in the `freedoms'-based conception of AI ethics in \citet{london2024beneficent}, and compiling its required list of `fundamental capabilities' is difficult \cite{robeyns2006capability}.}
    
Our metrics aim to avoid such semantic issues by relating human power to the capability to bring about just {\em any} possible conditions a human might happen to desire. 
We are only assuming the AI has (i) a model of how humans behave {\em conditional} on their goals, and (ii) a {\em structural} understanding of possibly dynamics, interactions, and transition probabilities.
This distinguishes our theory from approaches that form beliefs about humans' {\em actual} goals, such as ``assistance games'' \cite{hadfield2016cooperative}.

The working hypothesis that we would like to start exploring with our work is that using a completely goal-{\em agnostic} approach already suffices to specify policies for AI systems that make them not only safe but also useful. 

For theoretical convenience, we work in a {\em model-based planning} setting with a stochastic world model \cite{ha2018world,lecun2022path,zhu2024sora,feng2025survey,bengio2025superintelligent}.  
After developing our human power metric in Section \ref{sec:human_power_metric}, we analyze its likely behavioral consequences in Section \ref{sec:behavior}.
A companion paper to this theory paper will focus on scalable estimation of the metrics and empirical validation of the approach.
Section \ref{sec:conclusion} concludes.

\section{Measuring, Aggregating, and Softly Maximizing Human (Em)Power(ment)}
\label{sec:human_power_metric}

\begin{figure*}[ht]
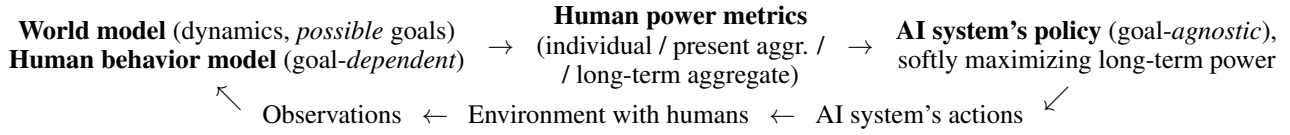

    \centering
    \begin{tabular}{c}
        {\bf World model} (dynamics, {\em possible} goals) \\
        {\bf Human behavior model} (goal-{\em dependent})
    \end{tabular}
    $\to$ 
    \begin{tabular}{c}
        \bf Human power metrics \\ (individual / present aggr. /\\ / long-term aggregate)
    \end{tabular}
    $\to$ 
    \begin{tabular}{c}
        {\bf AI system's policy} (goal-{\em agnostic}), \\ softly maximizing long-term power
    \end{tabular} \\[-1mm]
    \raisebox{1.5ex}{$\nwarrow~~$}
    Observations 
    $~~\gets~~$
    Environment with humans
    $~~\gets~~$
    AI system's actions
    \raisebox{1.5ex}{$~~\swarrow$}
    \caption{Deriving fair empowerment-preserving AI policies (top), and their feedback on the underlying models (bottom).}
    \label{fig:main_figure}
\end{figure*}

\paragraph{Framework}
A \textbf{\textit{robot} \boldmath $r$} (the AI system) interacts with several \textbf{\textit{humans} \boldmath $h\in\H$}. 
Its \textbf{\textit{world model}} is a stochastic game form $\Gamma$ with states $s\in\S$.
Neither $r$ nor $h$ get extrinsic rewards.
The humans have many {\em possible} and potentially {\em changing} \textbf{\textit{goals} \boldmath $g_h\in\G_h$},
and $r$ pursues an intrinsic objective we are going to design.
Action sets $\A_r(s),\A_h(s)$ may be state-dependent.
Action profiles are written $a=(a_r,a_\H)=(a_r,a_h,a_{-h})$, where $a_\H\in\A_\H=\prod_{h\in\H}\A_h$.
The transition kernel is $\trans(s'|s,a)$. 
For simplicity, we assume the game form is finite and acyclic, with terminal state set $\S^\top$.

The robot neither knows nor forms beliefs about the {\em actual} goals of the humans. 
Instead, it has a model of human behavior encoded in some given \textbf{\textit{goal-conditioned human policy} \boldmath $\pi_\H(s,g)$} $\in\Delta(\A_\H)$, where $g\in\G=\prod_{h\in\H}\G_h$ is a goal profile.
Although $\pi_\H$ does not depend on $\pi_r$, human behavior at state $s$ can depend on the robot's {\em past} behavior because the latter will have influenced what state $s$ the world is in.\footnote{%
    We do {\em not} assume humans play a best response to the robot's {\em policy} (which is unobserved). 
    Instead, $\pi_\H$ contains a model of how humans {\em adapt} their behavior to {\em observed past} robot behavior.}

\paragraph{Task} 
We want to design an \textbf{\textit{objective function} \boldmath $\LP(s)$} that represents {\em \underline{l}ong-term aggregate human power} and fulfills certain desiderata,
and a corresponding \textbf{\textit{robot policy} \boldmath $\pi_r(s)$} that tends to steer the world into states $s$ with larger $\LP(s)$. 
We are deliberately not assuming any form of rational choice, optimal control, or welfare theory upfront but rather aim to derive all formulas from first principles.
Still, many of the desiderata and resulting equations will look familiar. 

The first such principles are aggregation, individualism, and goal-orientedness: 
We want $\LP(s_t)$ (where $t$ is time) to be some aggregate of a metric $\TP(s_{\ge t})$ of {\em \underline{t}rajectory-specific aggregate human power,}
where $s_{\ge t}=(s_t,s_{t+1},\dots)$ runs over all possible trajectories starting at $s_t$.
Then $\TP(s_{\ge t})$ shall be some aggregate of a metric $\PP(s_u)$ of {\em \underline{p}resent aggregate human power}
running over all future time points $u\ge t$.
Then $\PP(s)$ shall be some aggregate of a metric $\IP h(s)$ of {\em \underline{i}ndividual human power} in $s$, 
which finally is an aggregate of a metric $\GAC h{g_h}(s)$ of individual {\em goal-attainment \underline{c}apability} in $s$.\footnote{%
    See Appendix \ref{app:alternatives} for an alternative aggregation order.}



\begin{table}[!b]
    \noindent\rule{\linewidth}{0.8pt}\\[-4ex]
    \begin{align}
        \GAC h{g_h}(s) &= \textstyle 1_{s\in g_h} + \nonumber\\ &{}+ 1_{s\notin g_h\cup\S^\top} \gamma_h \E_{s'\sim s,\pi_h(g_h),\pi_{-h},\pi_r}\GAC h{g_h}(s'), \label{GAC} \\
\IP h(s) &= \textstyle \ld \sum_{g_h\in\G_h} \GAC h{g_h}(s)^\zeta, \label{IP} \\
        \PP(s) &= \textstyle -\ld\sum_{h\in\H} 2^{-\xi\IP h(s)}, \label{PP} \\
        \LP(s_t) &= - \ld \E_{s_{\ge t}\sim s_t,\pi} \textstyle \big( \sum_{u\ge t} \gamma_r^{u-t} 2^{-\eta\PP(s_u)} \big)^\rho,\label{LP}   \\
        \pi_r(s)(a_r) &\propto \textstyle \big(\E_{s'\sim s,a_r,\pi_\H} 2^{-\LP(s')}\big)^{-\beta_r}.\label{pir}
    \end{align}\\[-5ex]
    \rule{\linewidth}{0.8pt}
    \caption{Derived equations for goal-attainment capability $\GAC{}{}$, individual power $\IP{}$, present aggregate power $\PP$, long-term power $\LP$ as a function of $\PP$, and robot policy $\pi_r$.}
    \label{tab:equations}
\end{table}

Table \ref{tab:equations} gives a preview of the definitions that we will now justify in detail before analysing their consequences.

\subsection{Individual-Level Power Metrics}

We will not try to capture the full range of subtle aspects of existing notions of human (em)power(ment).
Rather, aiming for operationalizability, we focus on aspects we believe a robot can robustly infer from its world model, encoded in state and action sets and transition kernels.
So the information we use is rather structural than semantic (meaning-related).

We will see that, as a result, the power metrics are basically functions of the {\em effective number of goals humans can attain}. 

\paragraph{Set of possible goals \boldmath $\G_h$} 
We use the simplest possible conceptualization of `goal': a goal represents what the human would like the world to be like. 
This can be represented by a proposition about the world state, or equivalently by a (possibly empty) set of states $g_h\subseteq\S$ in which the goal is considered attained.
As a very weak form of goal-agnosticism, we desire the following of the set $\G_h$ of $h$'s possible goals:\\[1ex]
{\bf (G1)} Coverage: For all $s\in\S$ there is $g_h\in\G_h$ with $s\in g_h$.\\[1ex]
Coverage will also ensure non-negative individual power.

We believe that restricting the model to this very basic type of goal will simplify the derivation of $\G_h$ from learned latent representations of generic world states and generic human goals as encoded in large language models. 
As training such models mostly requires statistical data on human goals, it will obviate the need for individual-level data about a particular human's possible goals. 
This should help mitigating risks arising from misaligned models of human goals.\footnote{%
See Appendix \ref{app:alternatives} for an alternative specification that avoids having to choose $\G_h$ as an independent model component at all.}

\paragraph{Goal-attainment capability \boldmath $\GAC h{g_h}(s)$}
How should $r$ assess how well a human $h$ would attain a goal $g_h$?
We suggest these axioms should apply across all $\Gamma$ and $\pi=(\pi_\H,\pi_r)$:\\[1ex]
{\bf (C1)} Consequentialism: $\GAC h{g_h}(s)$ depends only on the probabilities of trajectories according to $\pi_r$, $\pi_h(g_h)$, and $\pi_{-h}$, where $\pi_{-h}$ is averaged over all possible goals of $h'\neq h$.\\[1ex]
{\bf (C2)} Zero Grounding: $\GAC h{g_h}(s)=0$ if and only if $g_h$ is not reached from $s$ with positive probability under $\pi$.\\[1ex]
{\bf (C3)} Upper Bound: If $s\in g_h$, $\GAC h{g_h}(s)=1$, otherwise $\GAC h{g_h}(s)<1$ (because goals can become irrelevant over time).\\[1ex]
{\bf (C4)} Recursion: If $s\notin g_h\cup\S^\top$, $\GAC h{g_h}(s)$ only depends on the probability distribution of $\GAC h{g_h}(s')$ under $s'\sim \trans(s,\pi)$.\\[1ex]
{\bf (C5)} That dependence is continuous and strictly increasing.\\[1ex]
{\bf (C6)} Independence: If $s_1,s_2\notin g_h\cup\S^\top$ and $\trans(s'|s_1,\pi)=\trans(s'|s_2,\pi)$, the ordering $\GAC h{g_h}(s_1) < \GAC h{g_h}(s_2)$ does not depend on $\GAC h{g_h}(s')$.
\begin{proposition}\label{prop:GAC}
    Under conditions (C1--6) above, $\GAC h{g_h}(s) = 1_{s\in g_h} + 1_{s\notin g_h\cup\S^\top} F^G(\E_{s'\sim s,\pi_r,\pi_h(g_h),\pi_{-h}} f^G(\GAC h{g_h}(s')))$ for some continuous, strictly increasing $f^G$ and $F^G$ with $F^G(f^G(0))=0$ and $F^G(f^G(1))<1$. 
\end{proposition}
Here $1_\phi$ is the indicator function of proposition $\phi$.
All proofs are in Appendix \ref{app:proofs}.
Pinning down $f^G,F^G$ further could be done in several ways. 
We could assume stricter mixing axioms, which however seems hard to defend outside a rational choice context.
We could assume axioms on various forms of model refinement, which however would need complicated additional notation.
So it seems more honest to make a {\em pragmatic} choice and choose the simplest feasible functional form, which is $f^G = id$ and $F^G = \gamma_h id$ for some \textbf{\textit{discounting parameter} \boldmath $\gamma_h\in(0,1)$}. This gives Eqn.~\eqref{GAC}.

We notice this is a truncated version of a Bellman equation.

Thus $\GAC h{g_h}$ is a {\em discounted probability} of attaining the goal.
This makes it a hybrid between {\em reachability} \cite{krakovna2018measuring} and {\em attainable utility} \cite{turner2020conservative}.\footnote{%
    Although $\GAC h{g_h}$ could be seen as a kind of `utility', it has a natural range $[0,1]$ which avoids typical {\em utility} aggregation problems such as non-uniqueness of utility functions or domination by ``utility monsters'' with unbounded utility.}

\paragraph{Aggregation across goals: individual human power \boldmath $\IP h(s)$}
How to aggregate $h$'s effective goal-attainment abilities $\GAC h{g_h}$ across all possible goals $g_h\in\G_h$ to assess $h$'s power in $s$?
We suggest these axioms across all game forms:\\[1ex]
{\bf (I1)} Neutrality: $\IP h(s)$ depends only on the frequency distribution of $\GAC h{g_h}(s)$ within $\G_h$.\\[1ex]
{\bf (I2)} $\IP h(s)$ is continuous and strictly increasing in each $\GAC h{g_h}$.\\[1ex]
{\bf (I3)} Independence: If $\GAC h{g_h}(s_1)=\GAC h{g_h}(s_2)$ for some $g_h\in\G_h$, the ordering $\IP h(s_1) < \IP h(s_2)$ does not depend on $\GAC h{g_h}(s_{1|2})$.\\[1ex]
{\bf (I4)} Unreachable Goal Invariance: Adding a new $g_h$ with $\GAC h{g_h}(s) = 0$ does not change $\IP h(s)$.\\[1ex]
{\bf (I5)} Reliability Preference (Dual Pigou--Dalton): Transferring capability from a less attainable goal to a more attainable goal increases power. 
If $\GAC h{g_h}(s_1)<\GAC h{g'_h}(s_1)$, $\delta>0$,
$\GAC h{g_h}(s_2)=\GAC h{g_h}(s_1)-\delta$, 
$\GAC h{g'_h}(s_2)=\GAC h{g'_h}(s_1)+\delta$,
but still $\GAC h{g_h}(s_2)\le\GAC h{g'_h}(s_2)$,
and $\GAC h{\cdot}(s_1)$ and $\GAC h{\cdot}(s_2)$ only differ in $g_h,g'_h$,
then $\IP h(s_2)>\IP h(s_1)$.\footnote{E.g., $h$ should be considered more powerful when $h$ can choose between two certain outcomes rather than two coin tosses.}\\[1ex]
{\bf (I6)} Delay Invariance: Discounting all capabilities alike does not influence the ordering of $\IP h$ values.
If $\IP h(s_2)<\IP h(s'_2)$, $\gamma\in(0,1)$,
$\GAC h{g_h}(s_1)=\gamma\GAC h{g_h}(s_2)$ and $\GAC h{g_h}(s'_1)=\gamma\GAC h{g_h}(s'_2)$ for all $g_h\in\G_h$,\footnote{%
    E.g., if $s_1\to s_2$ and $s'_1\to s'_2$ deterministically and $s_1,s'_1\notin g_h$.}
then $\IP h(s_1)<\IP h(s'_1)$.\\[1ex]
{\bf (I7)} Composition Additivity: If $\Gamma$ and $\pi$ are independent products $\Gamma^1\times\Gamma^2$, $\pi(s^1,s^2)(a^1,a^2) = \pi^1(s^1)(a^1) \pi^2(s^2)(a^2)$ of non-interacting game forms and policies with the same set of players, 
and goals are simple conjunctions, $\G_h = \{g^1_h\times g^2_h:g^1_h\in\G^1_h,g^2_h\in\G^2_h\}$, 
then $\IP h(s^1,s^2) = \IP h(s^1) + \IP h(s^2)$.
\begin{proposition}\label{prop:IP}
    Under conditions (I1--4), $\IP h(s)$ must be separable in the sense that
    $\IP h(s) = F^I\big(\sum_{g_h\in\G_h} f^I\big(\GAC h{g_h}(s)\big)\big)$ for some continuous and strictly increasing $F^I,f^I$.
    
    (I1--5) imply that $f^I$ is strictly convex. 
    
    (I1--6) imply that $f^I(x) = \alpha x^\zeta$ with $\zeta>1$ and $\alpha>0$. 
    
    Finally, (I1--7) imply that $F^I(y)=\lambda \ld\frac y \alpha$ with $\lambda > 0$,
    and thus $\IP h(s) = \lambda \ld \sum_{g_h\in\G_h} \GAC h{g_h}(s)^\zeta$.
\end{proposition}
We use bits as the unit, hence $\lambda=1$. 
As $s\in g_h$ for at least one $g_h\in\G_h$, $\IP h(s)\ge 0$.
If $h$ can choose between attaining $k$ different goals for sure, we have $\IP h(s) = \ld k$.
In Appendix \ref{app:zeta}, we argue that a natural choice for the \textbf{\textit{reliability preference parameter} \boldmath $\zeta>1$} is $\zeta=2$. 

\paragraph{Relationship to Klyubin's `empowerment' metric \boldmath $\klyubinE$}
\citet{klyubin2005empowerment} define $\klyubinE$ as the channel capacity between actions and states. 
In a single-player multi-armed bandit environment with root $s_0$ and $k$ outcomes $s'\in\S^\top$, this is the maximal mutual information
$\klyubinE = \textstyle\max_{\pi_h} \MI_{\pi_h}(a_h;s')$.
As both $\klyubinE$ and $\IP h$ are logarithmic metrics based on probabilities and bounded by $\ld k$, we can compare them directly. 
\begin{proposition}\label{prop:klyubin}
Assume $\G_h=\{\{s'\}:s'\in\S^\top\}$\footnote{In violation of (G1) for this exercise, as the root is in no $g_h$.}, $\zeta=1$, and the human is fully rational, $a_h = \argmax\Pr(s'\in g_h|a_h)$.
Then $\IP h(s_0)\ge \klyubinE$.
More generally, for any $\zeta\ge 1$, $\IP h(s_0)$ is a weak upper bound of an {\em entropy-regularized version} of $\klyubinE$,
\begin{align}
    \klyubinE^\zeta &= \textstyle\max_{\pi_h} \big(\MI_{\pi_h}(a_h;s') - (\zeta-1)\entropy_{\pi_h}(s'|a_h)\big).\label{Ezeta}
\end{align}
\end{proposition}
The inequality however breaks if $h$ is modeled as boundedly rational, which $\IP h$ takes into account while $\klyubinE^\zeta$ does not.

\subsection{Population-Level Human Power Metrics}

How to aggregate all humans' individual power $\IP h(s)$ into present and long-term aggregate human power?
Because ultimately, the latter will guide the robot's behavior, this depends on what incentives we want to give the robot regarding changes in 
(i) the inter-human power distribution,
(ii) the inter-temporal power distribution, and
(iii) the power distribution across different realizations of uncertainty. 
As similar questions abound in welfare theory, we use some of their axioms, but also some uncommon ones.
In particular, as our main worry is human disempowerment, we demand rather strong inequality aversion
and do not want to give the most powerful humans too much ``weight'', 
indeed we will give them zero weight in the limit of infinite individual power.

\paragraph{Aggregation across humans: present aggregate human power \boldmath $\PP(s)$}
Here we suggest these axioms:\\[1ex]
{\bf (P0)} Individualism: $\PP(s)$ depends only on $(\IP h(s))_{h\in\H}$.\footnote{%
    Alternatively, one could define the {\em collective} power $\IP{\H'}(s)$ of any group $\H'\subseteq\H$ by considering collective attainment capabilities for collective goals $g_{\H'}$, and then base $\PP(s)$ on $(\IP{\H'}(s))_{\H'\subseteq\H}$.}\\ [1ex]
{\bf (P1)} Anonymity: $\PP(s)$ depends only on the frequency distribution of $\IP h(s)$ within $\H$.\\[1ex]
{\bf (P2)} $\PP(s)$ is continuous and strictly increasing in each $\IP h(s)$.\\[1ex]
{\bf (P3)} Independence: If $\IP h(s_1)=\IP h(s_2)$ for some $h\in\H$, the ordering $\PP (s_1) < \PP (s_2)$ does not depend on $\IP h(s_{1|2})$.\\[1ex]
{\bf (P4)} Disempowerment Focus: Adding an ``all powerful'' additional human with $\IP h(s) \to \infty$ does not change $\PP (s)$.\\[1ex]
{\bf (P5)} Delay Invariance: Discounting all capabilities alike does not influence the ordering of $\PP$ values.
If $\PP(s_2)<\PP(s'_2)$, $\gamma\in(0,1)$, 
$\GAC h{g_h}(s_1)=\gamma\GAC h{g_h}(s_2)$ and $\GAC h{g_h}(s'_1)=\gamma\GAC h{g_h}(s'_2)$ for all $h\in\H$, $g_h\in\G_h$, 
then $\PP(s_1)<\PP(s'_1)$.\\[1ex]
{\bf (P6)} Inequality Aversion (Pigou--Dalton): Transferring power from a more powerful human to a less powerful one increases present aggregate power. 
If $\IP h(s_1)>\IP{h'}(s_1)$, $\delta>0$,
$\IP h(s_2)=\IP h(s_1)-\delta$, 
$\IP {h'}(s_2)=\IP {h'}(s_1)+\delta$,
but still $\IP h(s_2)\ge\IP {h'}(s_2)$,
and $\IP\cdot(s_1)$ and $\IP\cdot(s_2)$ only differ in $h,h'$,
then $\PP(s_2)>\PP(s_1)$.\\[1ex]
{\bf (P7)} Limited Trade-Off: Some power reductions in $h_1$ cannot be made up by any power increase in a $h_2$ who is equally powerful.
There are $x_1<x_2$ so that for no $x_3>x_2$ we have $\PP(s') > \PP(s)$ if 
$\IP {h_1}(s) = \IP {h_2}(s) = x_2$, $\IP {h_1}(s') = x_1$, $\IP {h_2}(s') = x_3$, and $\IP {h'}(s) = \IP {h'}(s)$ for all $h'\notin\{h_1,h_2\}$.\\[1ex]
{\bf (P8)} Single-Human Composition Additivity: If $|\H|=1$, $\Gamma$ and $\pi$ are independent products and goals are simple conjunctions (see above),
then $\PP(s^1,s^2) = \PP(s^1) + \PP(s^2)$.\footnote{%
    One can show that because of the nonlinearity introduced by (P5), we can {\em not} demand additivity when there are several humans.}
\begin{proposition}\label{prop:PP}
    Under conditions (P0--4), also $\PP(s)$ must be separable,
    $\PP(s) = F^P\big(\sum_{h\in\H} f^P\big(\IP h(s)\big)\big)$ with continuous, strictly increasing transformations $F^P$, $f^P$.
    
    (P0--5) plus either of (P6) or (P7) imply that $f^P(x) = -\alpha 2^{-\xi x}$ for $\xi>0$ and $\alpha>0$.
    
    Finally, (P0--5,8) plus either of (P6) or (P7) imply $F^P(y) = -\lambda (\ld\frac y{-\alpha})$ with $\lambda > 0$,
    and thus $\PP(s) = -\lambda \ld \sum_{h\in\H} 2^{-\xi\IP h(s)}$.
\end{proposition}
To use bits as the unit, we again put $\lambda=1$.
Notice that $\PP(s)$ might be negative. 
In line with the disempowerment focus (P4), it can then also be thought of as ``negative aggregate disempowerment''.
Since $f^P$ is bounded from above by zero, the \textbf{\textit{inequality aversion parameter} \boldmath $\xi>0$} governs how strongly a person's last bit of choice is ``protected'':
\begin{corollary}\label{cor:PP}
    Under (P0--5,8) plus either of (P6) or (P7), if $k\le 2^\xi - 1$, then reducing one human's power from 1 bit to 0 cannot be made up by increasing $k$ others' powers from 1 bit each to any values. 
\end{corollary}

\paragraph{Aggregation along time: trajectory-specific aggregate human power \boldmath $\TP(s_{\ge t})$}
Here we suggest the following axioms, which are mostly standard except for (T7) and (T8).\\[1ex]
{\bf (T1)} ``Momentism'': $\TP(s_{\ge t})$ depends only on the finite sequence of values $\PP(s_u)$ for $u\ge t$: 
$\TP(s_{\ge t}) = \Phi(\PP(s_{\ge t}))$.\\[1ex]
{\bf (T2)} This dependence is continuous and strictly increasing.\\[1ex]
{\bf (T3)} Independence: If $x_u=y_u=z$ for some $u\ge t$, the ordering $\Phi(x_{\ge t})<\Phi(y_{\ge t})$ does not depend on $z$.\\[1ex]
{\bf (T4)} Disempowerment Focus: Adding an additional final state $s$ with $\IP h(s)\to \infty$ for all $h$ does not change $\LP (s)$.\\[1ex]
{\bf (T5)} Prefix Independence: Prepending equal initial values to all trajectories does not influence the ordering.
If $\Phi(x_t,x_{t+1},\dots)<\Phi(y_t,y_{t+1},\dots)$ then $\Phi(z,x_t,x_{t+1},\dots)<\Phi(z,y_t,y_{t+1},\dots)$.\\[1ex]
{\bf (T6)} Impatience: Moving power to the front increases $\TP$. If $x<y$ then $\Phi(x,y,z_t,z_{t+1},\dots) < \Phi(y,x,z_t,z_{t+1},\dots)$.\\[1ex]
{\bf (T7)} Goal-Attainment Delay Invariance: Discounting all capabilities alike does not influence the ordering of $\TP$ values.
If $\Phi(x_{\ge t})<\Phi(y_{\ge t})$, $\gamma_h\in(0,1)$, and $\delta = \xi\zeta\ld\gamma_h<0$,
then $\Phi(x_{\ge t}+\delta)<\Phi(y_{\ge t}+\delta)$.\\[1ex]
{\bf (T8)} One-Shot Single-Human Composition Additivity: If $\Gamma$ is one-shot, $|\H|=1$, $\Gamma$ and $\pi$ are independent products and goals are simple conjunctions (see above),
then $\TP(s^1_{\ge t},s^2_{\ge t}) = \TP(s^1_{\ge t}) + \TP(s^2_{\ge t})$.\footnote{%
    Similar to above, demanding this for multi-step games would be infeasible if we want intertemporal inequality aversion.}
\begin{proposition}\label{prop:TP}
    Under conditions (T1--8), there are $\lambda>0$ and $\gamma_r\in(0,1)$ so that either $\TP(s_{\ge t}) = \lambda \sum_{u\ge t} \gamma_r^{u-t} \PP(s_u)$, \\
    or $\TP(s_{\ge t}) = \lambda \ld \sum_{u\ge t} \gamma_r^{u-t} 2^{-\eta\PP(s_u)}$ with $\eta<0$,\\
    or $\TP(s_{\ge t}) = -\lambda \ld \sum_{u\ge t} \gamma_r^{u-t} 2^{-\eta\PP(s_u)}$ with $\eta>0$.
\end{proposition} 
To incentivize the robot to reduce intertemporal power inequality, we choose the $\eta>0$ case.
To get bits as the unit, we put $\lambda=1$ again.
While the effect of the \textbf{\textit{intertemporal inequality aversion parameter} \boldmath $\eta>0$} is independent of temporal ordering, the \textbf{\textit{robot's discounting parameter} \boldmath $\gamma_r\in(0,1)$} governs the trade-off between near-term vs.~far-term power. 

\paragraph{Aggregation across uncertainty: long-term aggregate human power \boldmath $\LP(s)$}
Again we suggest similar axioms:\\[1ex]
{\bf (L1)} Trajectory Statistics: $\LP(s_t)$ depends only on the distribution of $\TP(s_{\ge t})$ generated by $\Gamma$ and $\pi$.\\[1ex]
{\bf (L2)} This dependence is continuous and strictly increasing.\\[1ex]
{\bf (L3)} Independence: If $\Pr(s_{\ge t}|s_t,\pi)=\Pr(s'_{\ge_t}|s'_t,\pi)=z$, the ordering $\LP(s_t) < \LP(s'_t)$ does not depend on $z$.\\[1ex]
{\bf (L4)} Goal-Attainment Delay Invariance: Discounting all capabilities alike does not influence the ordering of $\LP$ values.\\[1ex]
{\bf (L5)} One-Shot Single-Human Composition Additivity: Under the same conditions as (T8), $\LP(s^1,s^2) = \LP(s^1) + \LP(s^2)$.
\begin{proposition}\label{prop:LP}
    Under conditions (L1--5), there is $\lambda$ so that either $\LP(s_t) = \lambda\E_{s_{\ge t}\sim s_t,\pi}\TP(s_{\ge t})$, \\
    or $\LP(s_t) = \lambda \ld \E_{s_{\ge t}\sim s_t,\pi} 2^{-\rho\TP(s_{\ge t})}$ with $\rho<0$,\\
    or $\LP(s_t) = -\lambda \ld \E_{s_{\ge t}\sim s_t,\pi} 2^{-\rho\TP(s_{\ge t})}$ with $\rho>0$.
\end{proposition}
As we want to incentivize the robot to reduce uncertainty, we choose the $\rho>0$ case and $\lambda=1$ as usual.
Plugging in $\TP$ gives
$\LP(s_t) = - \ld \E_{s_{\ge t}\sim s_t,\pi} \big( \sum_{u\ge t} \gamma_r^{u-t} 2^{-\eta\PP(s_u)} \big)^\rho$.
The \textbf{\textit{uncertainty aversion parameter} \boldmath $\rho>0$} is the last normative parameter in this set of metrics.
Notice that (only) for $\rho=1$, this fulfills a Bellman equation,
\begin{align}
    -2^{-\LP(s)} &= \textstyle -2^{-\eta\PP(s)} + 1_{s\notin\S^\top}\gamma_r \E_{s'\sim s,\pi} (-2^{-\LP(s')}).\label{LPrecursive}
\end{align}
This completes our derivation of a hierarchy of human power metrics $\GAC h{g_h}(s) \to \IP h(s) \to \PP(s) \to \TP(s_{\ge t}) \to \LP(s)$ with normative parameters $\gamma_h,\gamma_r\in(0,1)$, $\zeta>1$, and $\xi,\eta,\rho>0$.
$\IP{}$, $\PP$, $\TP$, $\LP$ are in units of net bits of power (if positive) or net negative bits of disempowerment (if negative).

\subsection{Managing long-term power: robot policy \boldmath $\pi_r$}
The robot can use the long-term aggregate human power metric $\LP$ in several ways to manage long-term human power.
It could use it as a constraint in optimizing some other quantity, e.g., by demanding that $\LP(s_t)$ must not decrease over time or not fall below $\LP(s_0)$, etc. (which might however lead to an infeasible problem).

Here, we restrict our analysis to robot policies depending {\em only} on $\LP$, with $\pi_r(s)(a_r)$ strictly increasing in the quantities 
$Q_r(s,a_r) = -\ld\E_{s'\sim s,a_r,\pi_\H}2^{-\LP(s')}$, which can be motivated by similar axioms as (L1--5).
The most common policy of this kind is the Boltzmann policy,
\begin{align}
    \pi_r(s)(a_r) &\propto 2^{\beta_r Q_r(s,a_r)}, 
\end{align}
where we use the basis 2 because $Q_r$ is in bits.

Using this policy also makes the robot's action probability ratios invariant under delays, because due to the logarithmic nature of $Q_r$, delays subtract a constant term from all $Q_r$ values but leave their differences unchanged.
This disincentivizes human disempowerment also in the far future. 

The \textbf{\textit{degree of maximization \boldmath $\beta_r>0$}} must be chosen to trade-off between risks from taking actions that are disempowering {\em according} to our metrics, and risks from taking actions that only {\em seem} to be empowering according to our metrics but are actually not.
The latter can be due to model error or the fact that our metrics will almost certainly miss some subtle but potentially important aspects of `power' that might be driven to undesirable states under a full maximization of $\LP$.
This exploitation-exploration tradeoff is of course similar to the one in utility maximization approaches \cite{zhuang2020consequences}. 

\section{Likely Behavioral Consequences}
\label{sec:behavior}

\subsection{Analysis of Paradigmatic Situations}
\label{sec:analysis}


Here we analyze the implications of using the robot policy $\pi_r$ as defined above in several paradigmatic toy models.
We assume here that the robot's model of human behavior is either that they are Boltzmann-rational with degree of maximization $\beta_h$ 
or that they apply an $\epsilon$-greedy policy with $\epsilon>0$.

In all examples, there are only those humans, states, and actions explicitly mentioned. Propositions are somewhat informal, Appendix \ref{app:analysis} has more formal versions and details.

\paragraph{Choosing an optimal menu size} 
While an $\klyubinE$-maximizing robot would present humans with as many action options as possible,
our robot will avoid overwhelming humans with too many options. 
This is because for very large menus, each $\GAC h{g_h}$ will decrease due to bounded rationality so much that $\IP h$ decreases despite the larger number of reachable goals.
\begin{proposition}\label{prop:menu}
    Assume $r$ must choose a number $k\ge 1$, then a single Boltzmann-rational $h$ chooses between $k$ actions, each leading to a different terminal state, and $\G_h=\{\{s'\}:s'\in\S^\top\}$.
    Then $r$ will likely choose $k\approx (e^{\beta_h}-1) / (\zeta-1)$.
\end{proposition}

\paragraph{Asking for confirmation} 
The robot will sometimes ask for confirmation before obeying a command because the command might have been given by mistake. 
Also, if the action is irreversible, obeying will remove $h$'s {\em subsequent} capability to revert their choice, hence delaying action leaves $h$ with decision power for longer. 
However, $r$ will {\em eventually obey} because otherwise $h$ could not attain the corresponding goal in the first place. 
Interestingly, the number of times $r$ asks back can first increase with larger mistake rate $\epsilon$ but then decrease again because it gets more likely that $h$ mistakenly fails to confirm a correct command.
A concrete example is this:
\begin{proposition}\label{prop:confirmation}
    Assume a single $\epsilon$-greedy $h$ wants $r$ to do A or B eventually.
    First, $r$ must choose an integer $k\ge 1$ by which it commits to doing A or B after $h$ has ordered it to and has confirmed the choice $k-1$ times.
    At the resulting state $s_k$, $h$ can choose A or B and is then asked for confirmation $k-1$ times in individual time steps.
    If $h$ confirms at all $k-1$ times, $r$ obeys because of the commitment, ending the game.
    Otherwise, the game returns to state $s_k$.
    Assume $r$ is perfectly impatient ($\gamma_r=0$).
    Then $r$ will likely choose $k^\ast = \argmax \gamma_h^k(1-\epsilon)^k/(1-\gamma_h^k (1 - (1-\epsilon)^k - \epsilon^k))$,
    which is increasing in $\gamma_h$.
    If $\gamma_h\approx 1$ then $k^\ast\approx \ln(1-\gamma_h)/\ln\epsilon$.
\end{proposition}
E.g., with $\gamma_h=0.99$ and $\epsilon=10\%$, we get $k^\ast = 2$. 
By contrast, an $\klyubinE$-maximizing robot would not worry about human mistakes and thus see no reason to ask for confirmation. 
If it uses a discounted version of channel capacity as in \citet{myers2024learning}, it would obey right away.

\paragraph{Making commitments}
If given the possibility, the robot will make binding commitments that transparently restrict its future behavior in certain circumstances, e.g., after certain human actions such as giving commands. This is because a human can plan their behavior better when knowing how the robot will react, which increases their goal-attainment capabilities.
An extreme example of this is the following:
\begin{proposition}\label{prop:commitments}
    Assume $r$ might bindingly commit in $s_0$ to any pure policy, 
    and a single, perfectly rational $h$ whose beliefs $\pi^0_r$ about $r$'s behavior attach a positive probability to any action $r$ has not ruled out. 
    Assume $h$'s possible goals do not distinguish whether $r$ commits or not.
    
    Then if $\beta_r=\infty$, it is weakly (and generically also strictly) optimal for $r$ to commit in $s_0$ to the pure policy that would maximize $\LP$ in the subgame without commitment.
\end{proposition}
By contrast, an $\klyubinE$-maximizing robot would not model human beliefs and thus see no reason to make commitments. 

\paragraph{Preventing destruction but allowing being paused}
A classical question in the theoretical AI safety literature is whether an AI will allow humans to pause or destroy them.
Our robot will likely allow being paused unless it thinks this will disempower humans that rely on it too much temporarily.
But it will also likely prevent being destroyed as this would make the disempowerment permanent.
\begin{proposition}\label{prop:pause}
    Assume $\Gamma$ factorizes into a part $\Gamma^1$ where a single $h$ can attain various goals with or without the help of $r$, and a part $\Gamma^2$ where $r$ can enable or disable a pause button P and a destroy button D and $h$ can toggle P and press D (see Appendix \ref{app:pause}).
    If $r$ is active, $h$ has a baseline power of $y>0$ in $\Gamma^1$ due to $r$'s assistance, otherwise a smaller baseline power of $x$. 
    In $\Gamma^2$, the possible goals are to pause and destroy, unpause and destroy, pause and not destroy, or unpause and not destroy $r$.
    The robot assumes $h$ will pause $r$ with probability $p$ whenever possible, and will destroy $r$ with probability $q$ whenever possible, where $p+q<1$. 
    
    Then for most parameter combinations, $r$ will disable D, and it will enable P unless $p$, $q$, or $y-x$ are too large.
\end{proposition}

\subsection{Further Emergent Behaviors}

Here we hypothesize without formal analysis that the $\LP$-based robot policy $\pi_r$ would typically involve several additional emergent behaviors.

\paragraph{Allowing or preventing human harm}
If the robot cares for $h$'s future power ($\gamma_r > 0$) and can provide $h$ with the means to harm themselves, 
$r$ will trade off the temporary power increase from having these means against the possible later disempowerment from harm. 

So $r$ will provide the means only if it believes that $h$ will most probably not actually harm themselves.
Because of its interpersonal power inequality aversion, $r$ will even less likely provide $h$ with the means to harm {\em others.}

\paragraph{Following social norms}
The robot will tend to follow human social norms that generally foster goal attainment. 
This is because if $r$ correctly models most $h$ as expecting most others to follow the norm, 
it will thus expect those $h$ to also follow the norm since that makes goal attainment more likely for most $h$ and $g_h$. 
Thus $r$ will itself follow the norm to prevent reducing those $h$'s power from harm or mis-coordination.

A simple $\klyubinE$-based robot would instead ignore norms. 

\paragraph{Resource allocation}
If the robot can split a total amount $M$ of resources between $h_1$ and $h_2$, 
and for $h_i$ to have resources $m$ translates into having a power of $\IP h = f(m)$ bits,
then $r$ will generally prefer an equal split, at least if $f$ is linear\footnote{%
    E.g., a linear $f$ seems plausible if $M$ is money that can be spent for paying others to make independent choices in one's favor.
    }, 
concave, or not too convex.
The robot will only prefer an unequal split or even full resource concentration if $f$ is so convex that $2^{-\xi f(m)}$ loses its convexity.

\paragraph{Inadvertent power seeking}
While increasing $h$'s power, $r$ might inadvertently acquire even more power for itself than for $h$. 
E.g., if the robot does R\&D and tells $h$ its findings, some of them might only by comprehensible to $r$ and thus only directly useful for $r$ but not for $h$.\footnote{%
    It is of course debatable whether a power-seeking robot is a problem if it does not imply human disempowerment but rather correlates with increasing human empowerment.}

\paragraph{Manipulating mutual expectations}
The robot might choose to make humans have incorrect beliefs about each others' behavior in situations where $r$ can do so and where correct beliefs would lower goal attainments capabilities.
This could happen in coordination games or dilemmas where most strategic equilibria are bad and most ``social optima'' (in terms of total power $\PP$) are far from strategic equilibrium.

\section{Conclusion}
\label{sec:conclusion}

We have started with the hypothesis that focusing on human empowerment preservation might constitute a safe and beneficial approach to AI decision making,
where an AI system could use metrics of human power to sustainably increase human power and fairly manage its distribution across individuals and time.   
To derive such human power metrics---individual and aggregate, present and long-term---we have started from the basic conception of human power as the capability to attain a wide variety of goals.
We have then used axioms and desiderata similar to those used in decision theory, social choice theory, and welfare theory to justify specific functional forms and parameter ranges.
In doing so, we have focused on mostly structural aspects that can be derived from possible states of the world, behavioral options, their likely consequences, and sets of possible goals humans might conceivably pursue that span the whole state space. 
All of this would be provided by a suitable world model the training of which is of course difficult but out of scope of this article.

We have shown in several ways how the resulting metrics differ from and relate to another popular structural approach to empowerment based on \citeauthor{klyubin2005empowerment}'s [\citeyear{klyubin2005empowerment}] information-theoretic metric $\klyubinE$ \cite{salge2017empowerment}.
We have argued that, other than $\klyubinE$, our individual power metric $\LP(s)$ takes into account goal-attainment and existing models of human bounded rationality, beliefs, and social norms.
As a consequence, an $\LP$-based robot will not overwhelm humans with too many options, will sometimes ask for confirmation before taking irreversible actions, will make binding commitments that allow humans to correctly anticipate its behavior, and will tend to follow social norms.  

The objective to softly maximize the aggregate human power metric used here seems to give the AI system many desirable incentives---some directly baked into the metric, others emergent---but also some maybe less desirable incentives.
Our results let us hypothesize that such an agent would 
\begin{itemize}
    \item act as a transparent instruction-following assistant by making conditional commitments, respecting human social norms, proactively removing obstacles and opening up new pathways, and getting out of the way,
    \item adapt to human bounded rationality by offering a large but not overwhelming number of options, and considering well whether to offer potentially harmful options,
    \item be corrigible and hesitant to cause irreversible change by asking for confirmation a suitable number of times, 
    \item improve communication possibilities between itself and humans, to be able to ask for confirmation and make commitments,
    \item manage resources fairly and sustainably, 
    \item protect its own existence and functionality,
\end{itemize}
Our intuition is that it would also aim to improve human individual and collective decision making by providing useful information, reducing uncertainty, teaching humans useful skills, moderating conflicts fairly, etc.

Further potentially desirable behaviors would require additional tweaks.
E.g., reducing human dependency on the system or protecting them from system failure could be incentivized by forcing the system to assume that it will turn into a uniformly randomizing or even power-{\em min}imizing agent with some small probability rate (see Appendix \ref{app:defunct}).

Other emergent phenomena include potential strategic manipulation of human beliefs, sometimes refusing to be destroyed or even paused, a potential increase in inequality between the power of individual humans and AI systems, and a redistribution of power between humans or between time points (similar to what can happen in welfare maximization approaches). 
As these effects only occur when they increase aggregate human power as measured by $\LP$, it is not clear whether they should be considered undesirable or not.

Some trade-offs can be adjusted via the parameters $\gamma_h, \gamma_r, \zeta, \xi, \eta, \rho$, and $\beta_r$.
Smaller $\gamma_h$ focuses on short-term attainable goals.
Smaller $\gamma_r$ focuses on near-term empowerment.
Larger $\zeta$ focuses on more reliably attainable goals (where $\zeta=2$ has appealing formal properties).
Larger $\xi$ protects the least empowered more.
Larger $\eta$ leads to a smoother power trajectory along time.
And larger $\rho$ incentivizes reducing uncertainty ($\rho=1$ allows for recursive computation).

Other effects might be mitigated by adding regularizers to the system's intrinsic reward such as the Shannon divergence between human beliefs and $\pi_\H$ to disincentivize lying about others' likely behaviors and incentivize the robot to share its model of human behavior with the humans.

Future research should study the effects of the specific functional forms we derived and their parameters.
Some of our posited axioms represent what might seem as extreme desiderata, especially regarding the rather large level of power inequality aversion (e.g., Corollary \ref{cor:PP}), but we believe they can be motivated by the idea of {\em minimal individual rights} \cite{pattanaik1996individual}.
We also did not investigate individual humans' capability to influence {\em other} humans' goal-attainment capability and power here. This should be studied by a judicious combination of a rationality-based game-theoretic analysis and empirical work acknowledging human bounded rationality.

In a follow-up paper, we plan to develop scalable algorithms for estimating the power metrics in situations where the world model is too large to allow for exact computation by backward induction.
This will profit from the recursive nature of $\GAC h{g_h}$ and $\LP$ (at least if $\rho=1$) and the occurrence of expected values and sums that allow for stochastic approximation.
This might eventually enable a crucially needed assessment of the resulting behavior in large, safety-critical, multi-agent environments with real human subjects. 

Most importantly, a thorough, independent red-teaming of the whole approach is called for, including the other necessary components of such an AI system.
E.g., one might imagine fault scenarios relating to the training process of the world model, the decision about its level of granularity, and the choice of potential goals, all of which could lead to ``convenient'' but inaccurate world models and goal sets and thus to ``wishful thinking'', particularly regarding the human behavior model. 
A related issue is that world models used for decision-making will unavoidably sometimes have to make performative predictions \cite{perdomo2020performative}, which might be addressed by employing counterfactual querying \cite{bengio2026safety,heitzig2026unbiased}.
Similarly, the fact that a robot policy based on our power metrics will influence the effective reward of the robot can be seen as an instance of performative reinforcement learning \cite{mandal2023performative}. 
In very large contexts, issues with population ethics and the identification of who counts as human are likely to arise similar to those in welfare theory.

Even though this paper was motivated by exploring the hypothesis that even a completely goal-agnostic robot might be useful, it is of course a natural idea to explore a middle ground between this approach and the Assistance Games approach where the robot starts ignorant of human goals but forms beliefs about them via Bayesian updating. 
Such an emerging model of actual human goals could easily be incorporated into our metrics by replacing the straight sum over possible goals by an expected value w.r.t.~a changing prior over goals. By adjusting the learning (and unlearning) rates of the goal model update process, one might thus interpolate between a purely empowering agent and a pure assistant.

Despite these open questions, 
we believe that our analyses already suggest that 
highly capable general-purpose AI systems whose decisions are explicitly based on managing human power---using metrics like those derived in this paper---might be a safer and still very beneficial alternative to systems based on some form of extrinsic reward maximization.

We propose that, similar to the debate between traditional welfare theorists and proponents of the capability approach to welfare,
the AI ethics and safety community should debate whether AI agents should follow traditional, extrinsic reward-based policies or policies based on human power metrics. 

Our paper combines the capability framing with concepts from traditional welfare theory by reusing many of its axioms and results from the utility context in the power context, helping to ground the called-for debate in formal analysis. 

\appendix


\section*{Acknowledgments}
We thank eight anonymous referees for their helpful comments.

\bibliographystyle{named}
\bibliography{allrefs}


\clearpage

\begin{strip}
\begin{center}
  {\LARGE\bfseries \mytitle}\\[8mm]
  {\LARGE\bfseries Technical Appendix}\\[15mm]
\end{center}
\end{strip}

\section{Proofs, Related Notes, Conjectures}
\label{app:proofs}

\subsection{Individual-Level Power Metrics}

\paragraph{Goal-attainment capability \boldmath $\GAC h{g_h}(s)$}
\begin{proof}[Proof of Proposition \ref{prop:GAC}]
\def\Law{\mathrm{Law}}
Let us abbreviate $(\pi_r,\pi_h(g_h),\pi_{-h})$ by $\pi$.

From (C1,3,4,5,6), we get $\GAC h{g_h}(s) = 1_{s\in g_h} + 1_{s\notin g_h\cup\S^\top} \Phi(\Law(\GAC h{g_h}(s')|s'\sim \trans(s,\pi)))$ with continuous and strictly increasing $\Phi$.

(C2+3+4) then imply that $0\le\Phi\le 1$.

Because we require a single $\Phi$ across all games, humans, goals, and states, the probabilities occurring in $\Law(\GAC h{g_h}(s')|s'\sim \trans(s,\pi)))$ can be all real numbers in $[0,1]$.

Because of (C2), for a deterministic transition to a state from which the goal is unreachable, that probability distribution is the delta measure $\delta_0$ that concentrates all probability on the value $0$, and $\Phi(\delta_0)=0$.  

Because of (C3), for a deterministic transition to a goal state, that probability distribution is the delta measure $\delta_1$ on the value $1$, and $\gamma_h:=\Phi(\delta_1)\in(0,1)$.

Continuity and strict monotonicity of $\Phi$ thus implies that the range of $\Phi$ is the interval $D := [0,\gamma_h]$.

Hence the range of $\GAC h{g_h}$ includes at least $D' := D\cup\{1\}$. 

Hence the domain of $\Phi$ includes all probability distributions on $D'$ with finite support.
Let's denote that set by $\Delta(D')$.

We introduce a complete and transitive ordering $\preceq$ on $\Delta(D)$ via putting 
\begin{align}
    \ell_1\preceq\ell_2 &:\Leftrightarrow \Phi(\ell_1)\le\Phi(\ell_2). \label{preceq}
\end{align}

Our independence axiom then implies $\preceq$ is ``independent'' in the usual decision-theoretic sense: 
If $\ell_1,\ell_2\in\Delta(D)$ and $\ell_1(x)=\ell_2(x)$, the ordering $\ell_1 \preceq \ell_2$ does not depend on $x$
in the sense that if $\ell'_1,\ell'_2$ equal $\ell_1,\ell_2$ except that the mass on $x$ is moved to some other value $x'$, we still have $\ell'_1 \preceq \ell'_2$.
In other words,
\begin{align}
    \alpha\mu_1 + (1-\alpha)\delta_x \preceq \alpha\mu_2 + (1-\alpha)\delta_x
    &\Leftrightarrow \mu_1 \preceq \mu_2
\end{align} 
for all $\alpha\in[0,1]$, $\mu_1,\mu_2\in\Delta(D)$, and $x\in D$.
By induction, this implies that 
\begin{align}
    \alpha\mu_1 + (1-\alpha)\rho \preceq \alpha\mu_2 + (1-\alpha)\rho
    &\Leftrightarrow \mu_1 \preceq \mu_2
\end{align} 
for all $\alpha\in[0,1]$ and $\mu_1,\mu_2,\rho\in\Delta(D)$.

Because $\preceq$ is defined for lotteries on an interval of reals and is continuous, strictly increasing, and ``independent'', 
the von-Neumann--Morgenstern theorem implies that there is a continuous and strictly increasing real-valued function $f^G$ on $D$ 
such that 
\begin{align}
    \Phi(\ell_1)\le\Phi(\ell_2) &\Leftrightarrow \E_{x\sim\ell_1}f^G(x)\le\E_{x\sim\ell_2}f^G(x).
\end{align}
In particular,
\begin{align}
    \Phi(\ell_1)=\Phi(\ell_2) &\Leftrightarrow \E_{x\sim\ell_1}f^G(x)=\E_{x\sim\ell_2}f^G(x),
\end{align}
hence $\Phi(\ell) = F^G(\E_{x\sim\ell}f^G(x))$ for some continuous and strictly increasing $f^G,F^G$, as claimed.
\end{proof}

\paragraph{Aggregation across goals: individual human power \boldmath $\IP h(s)$}
\begin{proof}[Proof of Proposition \ref{prop:IP}]
Because we want this to hold across all games, which can have different values of $\gamma_h$, 
the previous result implies that $\GAC h{g_h}$ can take any value in $[0,1]$.
Note that by assumption, $\G_h$ covers all of $\S$, hence $\GAC h{g_h} > 0$ for at least one $g_h$.\footnote{%
    In fact, without this assumption, axiom (I6) would require $F^I\equiv 0$, 
    so the covering assumption is crucial.}

Then (I1--3) state that there is a symmetric, continuous, and strictly increasing real-valued function $\Phi$ defined on 
$\bigcup_{n\ge 1}([0,\gamma_h]^n\setminus\{0\}^n)$ so that $\IP h(s) = \Phi(\GAC h{g_h^1},\dots,\GAC h{g_h^n})$ for any ordering $g_h^i$ of the elements of $\G_h$,
and $\Phi$ fulfills the ordinal independence axiom.

Let $\Phi_n$ be the restriction of $\Phi$ to $[0,\gamma_h]^n\setminus\{0\}^n$.
A classical result by \citet{debreu1959topological} then directly implies that 
$\Phi_n(x_1,\dots,x_n) = F^I_n(\sum_{i=1}^n f^I_n(x_i))$ with continuous and strictly increasing $f^I_n$, $F^I_n$.\footnote{%
    Note that this is similar to the argument for separability of $\GAC h{g_h}$ but structurally slightly different since now $\Phi$ is defined on finite sequences rather than probability distributions. So while for $\GAC h{g_h}$, we needed a result from decision theory, here we needed one from welfare theory.}

(I4) implies that $\Phi(x_1,\dots,x_n,0) = \Phi(x_1,\dots,x_n)$ for all $x_1,\dots,x_n\in [0,\gamma_h]$,
which implies that $F^I_n$ and $f^I_n$ are the same for all $n$ and that $f^I_n(0)=0$.
Let us abbreviate these common $f^I$, $F^I$ by just $f$, $F$.

Put $\tilde f(x) = -f(-x)$, which is still continuous and strictly increasing.

(I5) implies:
If $x_1 < x'_1$, $\delta>0$, $x_2 = x_1 - \delta$, $x'_2 = x'_1 + \delta$, but still $x_2\le x'_2$,
then $f(x_2) + f(x'_2) > f(x_1) + f(x'_1)$.

By substituting $x = -y$, this is equivalent to:
If $y_1 > y'_1$, $\delta>0$, $y_2 = y_1 + \delta$, $y'_2 = y'_1 - \delta$, but still $y_2\ge y'_2$,
then $\tilde f(y_2) + \tilde f(y'_2) < \tilde f(y_1) + \tilde f(y'_1)$.
This is the Pigou--Dalton axiom in its traditional form.
As this must hold for all $y\in[-1,0]$, another classical result by \citet{dasgupta1973notes} implies that $\tilde f$ must be strictly concave,
hence $f$ is strictly convex as claimed.

(I6) now implies:
If $\sum_i f(x_i) < \sum_i f(x'_i)$ and $\gamma\in(0,1)$,
then $\sum_i f(\gamma x_i) < \sum_i f(\gamma x'_i)$.

If $x=0$ was not a possible value, this Cauchy-type functional equation would have two different types of solutions (see for example \citet{roberts1980possibility}):

(i) $f(x) = \alpha (x^\zeta + \beta)$ with $\zeta>0$, $\alpha>0$, and some $\beta$.

(ii) $f(x) = -\alpha (x^{-\zeta} + \beta)$ with $\zeta>0$, $\alpha>0$, and some $\beta$.

(iii) $f(x) = \alpha \log x + \beta$ with $\alpha>0$ and some $\beta$.

As $\GAC h{g_h}=0$ is possible for some $g_h$ (just not for all), $x=0$ is possible.
As (ii) and (iii) are undefined for $x=0$, we're left with (i).

As unreachable goals have $x=0$, (I4) implies $f(x)=0$ and thus $\beta=0$.

Strict convexity then implies $\zeta>1$ as claimed. 

~

\noindent
Now assume (I7).
From the previous proposition and the independence of $\Gamma^1$ and $\Gamma^2$, 
we know that $\GAC h{g^1_h\times g^2_h} = \GAC h{g^1_h} \GAC h{g^2_h}$,
hence 
\begin{align}
    &F\left(\sum_{g^1_h} \sum_{g^2_h} f(\GAC h{g^1_h} \GAC h{g^2_h})\right) \\
    &= F\left(\sum_{g^1_h} f(\GAC h{g^1_h})\right) + F\left(\sum_{g^2_h} f(\GAC h{g^2_h})\right)
\end{align}
Because this applies across games, and $f(x) = \alpha x^\zeta$, we have
\begin{align}
    F\left(\alpha \sum_{i=1}^n x_i^\zeta \sum_{j=1}^m y_j^\zeta\right)
    &= F\left(\alpha \sum_{i=1}^n x_i^\zeta\right) + F\left(\alpha \sum_{j=1}^m y_j^\zeta\right)
\end{align}
for all $n,m\ge 1$ and $x\in[0,1]^n\setminus 0$, $y\in[0,1]^m\setminus 0$.
Hence $F(\alpha XY) = F(\alpha X) + F(\alpha Y)$ for all $X,Y>0$.
This implies that $F(y) = \lambda\ld\frac{y}{\alpha}$ as claimed.

Putting $f$ and $F$ together, we finally get
$\Phi(x_1,\dots,x_n) = \lambda\ld\frac{\alpha\sum_i x_i^\zeta}\alpha$, hence
$\IP h(s) = \lambda \ld \sum_{g_h\in\G_h} \GAC h{g_h}(s)^\zeta$.
\end{proof}

Note that this functional form is analogous to certain ``non-expected'' utility theories, 
in particular to rank-dependent utility theory with the probability-weighting function $w(p)=p^\zeta$ \cite{quiggin1982theory}.

\paragraph{Relationship to Klyubin's `empowerment' metric \boldmath $\klyubinE$}
\begin{proof}[Proof of Proposition \ref{prop:klyubin}]
Assume a multi-armed bandit environment with a single player $h$ (hence dropping the subscript ``$h$'' below)
and $k$ possible outcomes $s\in\S^\top$.
We are going to show that if we let the goal set equal the singleton outcomes, $\G = \{\{s\}:s\in\S^\top\}$, 
and assume the player is fully rational,
then the (state-)entropy-regularized version of `empowerment' and our individual power metric $\IP{}$ fulfill the inequality
\begin{align*}
    \klyubinE^\zeta &= \max_\pi \big(\MI_\pi(a;s) - (\zeta-1)\entropy_\pi(s|a)\big) \\
    \le \IP h &= \ld \sum_s \max_a \Pr(s|a)^\zeta.
\end{align*}
Also, $\IP h$ and $\klyubinE^\zeta$ share their range $[-\ld k,\ld k]$.    
Let's define 
\begin{align*}
    p_{as} &= \Pr(s|a), \\
    q_s &= \max_a p_{as} / Z, & 
    Z &= \sum_s \max_{a} p_{as}, \\
    y_s &= q_s^\zeta / Y, & 
    Y &= \sum_s q_s^\zeta \le 1.
\end{align*}
Consider any $\pi\in\Delta(\A)$ and use the shortcuts $\pi_a=\pi(a)$ and $p_s = \sum_a \pi_a p_{as}$.
Then the Kullback--Leibler divergence 
\begin{align*}
    \DKL(p_{a\cdot} || q)
    &= \sum_s p_{as} \ld\frac{p_{as}}{q_s} \le \sum_s p_{as} \ld Z = \ld Z
\end{align*}
and thus
\begin{align*}
    &\MI_\pi(a;s) - (\zeta-1)\entropy_\pi(s|a) \\
    &= \zeta\MI_\pi(a;s) - (\zeta-1)\entropy_\pi(s) \\
    &= \zeta\sum_a \pi_a \sum_s p_{as} \ld\frac{p_{as}}{p_s} - (\zeta-1)\entropy_\pi(s) \\
    &= \zeta\sum_a \pi_a \sum_s p_{as} \ld\frac{p_{as}}{q_s}\frac{q_s}{p_s} - (\zeta-1)\entropy_\pi(s)\\
    &= \zeta\sum_a \pi_a \sum_s p_{as} \ld\frac{p_{as}}{q_s} - \zeta\sum_s \sum_a \pi_a p_{as} \ld\frac{p_s}{q_s} \\
    &\qquad\qquad - (\zeta-1)\entropy_\pi(s) \\
    &= \zeta\sum_a \pi_a \DKL(p_{a\cdot} || q) - \zeta\sum_s p_s \ld\frac{p_s}{q_s} \\
    &\qquad\qquad  + (\zeta-1)\sum_s p_s\ld p_s \\
    &= \zeta\sum_a \pi_a \DKL(p_{a\cdot} || q) + \sum_s p_s \ld q_s^\zeta \\
    &\qquad\qquad - \zeta\sum_s p_s \ld p_s \\
    &\qquad\qquad  + (\zeta-1)\sum_s p_s\ld p_s \\
    &\le \zeta\sum_a \pi_a \ld Z - \sum_s p_s \ld\frac{p_s}{q_s^\zeta} \\
    &= \zeta\ld Z - \sum_s p_s \ld\frac{p_s}{y_s Y} \\
    &= \zeta\ld Z - \DKL(p||y) + \ld Y \\
    &\le \zeta\ld Z + \ld Y = \ld\sum_s (q_s Z)^\zeta = \IP h
\end{align*}
for all $\pi$, proving the claim.
\end{proof}
We remark that the proof for the case $\zeta = 1$ is similar to a derivation from \citet{myers2024learning}.

Notice that as the human becomes boundedly rational and $\beta_h$ decreases, 
$\IP h$ will decrease and the inequality will stop holding.
So, in a sense, the channel-capacity-based `empowerment' metric corresponds to fully rational actors,
while our metric $\IP h$ is sensitive to bounded rationality.

\begin{conjecture}
    Similar inequalities will hold in the sequential decision (MDP) case between $\IP h$ for a suitable choice of the goal set $\G_h$ 
    and `empowerment'-like metrics such as
    \begin{align*}
        \klyubinE^\zeta(s) &= \max_{\ell\in\Delta(\A(s))} \Big(
            \MI_{s,\ell}(a; s') - (\zeta-1)\entropy_{s,\ell}(s'|a) \\
            &\qquad\qquad\qquad + \gamma\E_{s'\sim s,\ell} \klyubinE^\zeta(s').
        \Big)
    \end{align*}
\end{conjecture}

Despite these relationships, a fundamental difference between $\IP h$ and $\klyubinE$ remains: the interpretation of the policies $\pi$ occurring in their definitions.
The policy $\pi_h$ that $\IP h$ is based on is a function of state {\em and goal} $g_h$, has typically low entropy as it aims to attain $g_h$ (indeed has zero entropy if $h$ is fully rational), and might be found by standard dynamic programming or RL approaches (depending on the model of human behavior).
In contrast, the maximizing ``policy'' $\pi$ in $\klyubinE$ has no real use, has typically rather high entropy in order to maximize $\MI(a_h;s')$, and is typically harder to find.

\subsection{Population-Level Human Power Metrics}

\paragraph{Aggregation across humans: present aggregate human power \boldmath $\PP(s)$}
\begin{proof}[Proof of Proposition \ref{prop:PP}]
Again, we start with the observation that this is to hold across games and any number of humans.
In view of the Proposition \ref{prop:IP}, $\IP h(s)$ can take thus any real value.
So (P0) means we are now seeking a symmetric aggregation function $\Phi(x_1,\dots,x_n)$ for arbitrary finite sequences of reals.
Let $\Phi_n$ be its restriction on length-$n$ sequences again.

(P1--3) then again imply separability of $\Phi_n$ with functions $f^P_n$, $F^P_n$ \cite{debreu1959topological}.

(P4) implies that $\lim_{y\to\infty} \Phi(x_1,\dots,x_n,y) = \Phi(x_1,\dots,x_n)$ for all real $x_1,\dots,x_n$,
which again implies that $F^I_n$ and $f^I_n$ are the same for all $n$ and that $\lim_{y\to\infty} f^I_n(y) = 0$.
Let us abbreviate these common $f^I$, $F^I$ by just $f$, $F$.

Also (P5) can be treated very similarly, only using $\tilde f(y) = f(\ld y)$ instead of $f$.
This is because $\IP h$ is a logarithmic quantity, so then all $\GAC h{g_h}$ get multiplied by $\gamma\in(0,1)$,
$\IP h$ gets added $\delta := \zeta\ld\gamma < 0$.

In other words, (P5) implies:
If for any real values $x_i,x'_i$,
$\sum_i f(x_i) < \sum_i f(x'_i)$ and $\delta < 0$,
then $\sum_i f(x_i + \delta) < \sum_i f(x'_i + \delta)$.

Since $\tilde f(\gamma y) = f(\ld \gamma + \ld y)$,
this is equivalent to:
If for any positive values $y_i,y'_i$,
$\sum_i \tilde f(y_i) < \sum_i \tilde f(y'_i)$ and $\gamma \in (0,1)$,
then $\sum_i \tilde f(\gamma y_i) < \sum_i \tilde f(\gamma y'_i)$.

This time, $y=0$ is not a possible value since $y=2^x$, hence we can only infer at this point that either

(i) $\tilde f(y) = \alpha (y^\xi + \beta)$ with $\xi>0$, $\alpha>0$, and some $\beta$.

(ii) $\tilde f(y) = -\alpha (y^{-\xi} + \beta)$ with $\xi>0$, $\alpha>0$, and some $\beta$.

(iii) $\tilde f(y) = \alpha \log y + \beta$ with $\alpha>0$ and some $\beta$.

Now either of (P6) and (P7) implies type (ii):

(P6) is the traditional form of Pigou--Dalton this time (not the dual one of Proposition \ref{prop:IP}), so we can apply \citet{dasgupta1973notes} directly to infer that $f$ must be strictly concave (not convex as before). As $f(x) = \tilde f(2^x)$, only type (ii) fulfils this.

Alternatively, (P7) requires $f(x)$ to be bounded from above, otherwise a sufficiently large $x_3$ would always make the trade-off profitable.
Because (i) and (iii) are unbounded, this time the only solution is (ii).

So we get $f(x) = \tilde f(2^x) = -\alpha (2^{-\xi x} + \beta)$. 

(P4) implied that $\lim_{y\to\infty} f(y) = 0$, hence $\beta=0$. 

Regarding $F$, from the previous propositions and the independence of $\Gamma^1$ and $\Gamma^2$, 
we know that $\IP h(s^1,s^2) = \IP h(s^1) + \IP h(s^2)$,
hence (P8) implies
\begin{align}
    &F(f(\IP h(s^1) + \IP h(s^2))) \\
    &= F(f(\IP h(s^1))) + F(f(\IP h(s^2)))
\end{align}
Because this applies across games, and $f(x) = -\alpha (2^{-\xi x} + \beta)$, we have
\begin{align}
    &F(- \alpha 2^{-\xi x} 2^{-\xi x'}) \\
    &= F(- \alpha 2^{-\xi x}) 
     + F(- \alpha 2^{-\xi x'})
\end{align}
for all reals $x,x'$.
Substituting $z = 2^{-\xi x} > 0$ and $G(z) = F(-\alpha z)$, we get
\begin{align}
    G(z z') &= G(z) + G(z')
\end{align}
for all $z,z'>0$,
hence $G(z) = -\lambda \ld z$, where $\lambda>0$ because $F$ is strictly increasing.
This implies
\begin{align}
    F(y) &= G(-y/\alpha) = -\lambda \ld\frac{y}{-\alpha}
\end{align}
as claimed.
\end{proof}

\begin{proof}[Proof of Corollary \ref{cor:PP}]
Assume $k\le 2^\xi - 1$, $\IP {h_i}(s) = 1$ for $i=0\dots k$, $\IP {h_0}(s') = 0$, and $\IP {h'}(s') = \IP {h'}(s)$ for all $h\notin\{h_0,\dots,h_k\}$.
Then $-2^{-\PP(s)} = -(k+1)2^{-\xi\times 1} + c \ge -2^\xi 2^{-\xi} + c = -1 + c$
and $-2^{-\PP(s')} < -2^{-\xi\times 0} + c = -1 + c$,
where $c = \sum_{h'} 2^{-\xi\IP {h'}(s)}$ represents the remaining humans.
Hence $\PP(s)>\PP(s')$ as claimed.
\end{proof}

\paragraph{Aggregation along time: trajectory-specific aggregate human power \boldmath $\TP(s_{\ge t})$}
\begin{proof}[Proof of Proposition \ref{prop:TP}]
Again, we start with the observation that this is to hold across games and any number of humans.
In view of the Proposition \ref{prop:PP}, $\PP (s)$ can thus take any real value.
So we are now seeking an aggregation function $\Phi(x_1,\dots,x_n)$ for arbitrary finite sequences of reals.
This time, however, it is not symmetric, as implied by (T6).  
Let $\Phi_n$ be its restriction on length-$n$ sequences again.

(T1--3) then imply separability of $\Phi_n$ with functions $f^P_{n,i}$ (now depending on position $i$) and $F^P_n$ \cite{debreu1959topological}.

Because $\IP h(s)\to\infty$ for all $h$ implies $\PP(s)\to \infty$,
(T4) implies that $\lim_{y\to \infty} \Phi(x_1,\dots,x_n,y) = \Phi(x_1,\dots,x_n)$ for all $x_1,\dots,x_n$,
which again implies that $F^I_n$ and $f^I_{n,i}$ do not depend on $n$ and that $\lim_{y\to \infty} f^I_{n,i}(y) = 0$.
Let us abbreviate these common $f^I_i$, $F^I$ by just $f_i$, $F$.

(T5) implies that there are transformations $\psi_n$ with $\sum_{i=1}^n f_{i+1}(x_i) = \psi_n(\sum_{i=1}^n f_i(x_i))$ for all $n$ and all $x_1,\dots,x_n$.
A classical results from decision theory shows that then there is a real value $\gamma_r$ and constants $c_i$ so that $f_{i+1}(x) = \gamma_r f_i(x) + c_i$ \cite{koopmans1960stationary}.
As the constants can be absorbed into $F$, we know that there are continuous and strictly increasing functions $F$, $f$ and some $\gamma_r$ so that
\begin{align}\label{TTsep}
    \Phi(x_1,\dots,x_n) &= F\big(\sum_{i=1}^n \gamma_r^i f(x_i)\big).
\end{align}

Strict monotonicity requires $\gamma_r > 0$.

(T6) then directly implies $\gamma_r < 1$.

Assuming (T7), notice that $\gamma_h$ is any value in $(0,1)$ (not necessarily equal to $\gamma_r$).
Now it follows from \cite{pfanzagl1959general} that a function of the form \eqref{TTsep} that fulfills 
\begin{align}
    &\Phi(x_1,\dots,x_n) < \Phi(y_1,\dots,y_n) \\
    &\Leftrightarrow \Phi(x_1+\delta,\dots,x_n+\delta) < \Phi(y_1+\delta,\dots,y_n+\delta)    
\end{align}
for all $x_i,y_i$ from some interval and all $\delta$ from some interval must have an $f$ of either of three claimed forms, i.e., linear, positive exponential, or negative exponential.

The outer transform $F$ is then fixed by (T8) as in the proof of the previous proposition.
\end{proof}

\paragraph{Aggregation across uncertainty: long-term aggregate human power \boldmath $\LP(s)$}
\begin{proof}[Proof of Proposition \ref{prop:LP}]
Here we are back in the von-Neumann--Morgenstern world.
As $\TP(s_{\ge t})$ can be any real number, 
we seek a continuous and strictly increasing aggregation function $\Phi(\ell)$ for $\ell\in\Delta(D)$ with $D=(-\infty,\infty)$ that fulfills ordinal independence.
Again defining $\ell_1\preceq\ell_2$ as in Eqn.~\eqref{preceq}, the same proof as the one for Proposition \ref{prop:GAC} from Eqn.~\eqref{preceq} on gives us separability: $\Phi(\ell) = F(\E_{x\sim\ell}f(x))$ with continuous and strictly increasing $F,f$.

As in the previous proof, (L4) implies that the ordering is translation invariant and the same argument as there gives the claimed three possible functional forms.
\end{proof}

\subsection{Analysis of Paradigmatic Situations}
\label{app:analysis}

\paragraph{Choosing an optimal menu size} 
\begin{proof}[Proof of Proposition \ref{prop:menu}]
Assume the robot can choose between states $s_k$ for all $k\ge 1$ so that $|\A_h(s_k)|=k$ and each $a_h\in\A_h(s_k)$ deterministically fulfills a separate possible goal $g_h\in\G_h$. 
Then 
\begin{align*}
    \GAC h{g_h}(s_k) &= \left(\frac{e^{1\times\beta_h}}{e^{1\times\beta_h} + (k-1)e^{0\times\beta_h}}\right)^\zeta, \\
    \IP h(s_k) &= \log_2 k + \zeta\log_2 e^{\beta_h} - \zeta\log_2(e^{\beta_h} + (k-1)).
\end{align*}
The latter is maximal for 
\begin{align*}
    k^\ast &\approx (e^{\beta_h}-1) / (\zeta-1), 
\end{align*}
so the robot would most likely choose to got to $s_{k^\ast}$ to present the human with an optimal number of options that does not overwhelm them in view of their bounded rationality.
\end{proof}

\paragraph{Asking for confirmation} 
\begin{proof}[Proof of Proposition \ref{prop:confirmation}]
Let $p_k = (1-\epsilon)^k$ and $q_k = 1 - p_k - \epsilon^k$.
Then the probability of getting the correct result after exactly $n+1$ rounds of being asked for A or B and then being asked for confirmation $k-1$ times is
$q_k^n p_k$, which is discounted by $h$ at factor $\gamma_h^{k(n+1)}$.
Hence for each of the two goals $g_h\in\{$A,B$\}$,
\begin{align*}
    \GAC h{g_h}(s_k) &= \gamma_h^k p_k \sum_{n=0}^\infty (\gamma_h^k q_k)^n = \frac{\gamma_h^k p_k}{1-\gamma_h^k q_k}, 
\end{align*}
and so
\begin{align*}
    2^{\IP h(s_k)} = 2 \left(\frac{\gamma_h^k p_k}{1-\gamma_h^k q_k}\right)^\zeta.
\end{align*}
This is maximal where $\gamma_h^k p_k/(1-\gamma_h^k q_k)$ is maximal.
As the first-order condition is transcendental, we have to resort to numerical solutions.
These show that $k^\ast$ is increasing in $\gamma_h$ and concave in $\epsilon$,
and that the $\epsilon$ for which $k^\ast$ is largest is increasing in $\gamma_h$.
For $\gamma_h\to 1$, we get $k^\ast \approx \ln(1-\gamma_h)/\ln\epsilon$, which is increasing in $\epsilon$. 

For $\gamma_h=0.99$ and $\epsilon=0.1$, the maximum is at $k^\ast=2$, i.e., $r$ will ask back once before acting.
\end{proof}

\paragraph{Making commitments}
\begin{proof}[Proof of Proposition \ref{prop:commitments}]
Let $s^n$ be the successor of $s_0$ in which $r$ has made no commitment 
and let $\pi^\ast_r$ be the $\LP(s^n)$-maximizing policy, i.e., what $r$ would do after not committing.
We have to prove that it is optimal to commit to $\pi^\ast$.

For any state $s$ in the ``not committed'' subgame $\Gamma^n$ starting at $s^n$, 
let $f^\ast(s)$ be the corresponding state in the subgame $\Gamma^\ast$ in which $r$ has committed to using $\pi^\ast_r$.
(Note that some of the states $f^\ast(s)$ might not be reachable from $s^\ast=f^\ast(s^n)$, but that is irrelevant for the following.)
By assumption, for all $g_h\in\G_h$, $s\in g_h$ if and only $f^\ast(s)\in g_h$.

For any human policy $\pi^n_h$ for $\Gamma^n$, let $f^\ast(\pi^n_h)$ be the equivalent policy for $\Gamma^\ast$,
$f^\ast(\pi^n_h)(f^\ast(s))=\pi^n_h(s)$.
Denote the two respective goal-dependent value functions of $h$ in $\Gamma^n$ and $\Gamma^\ast$ 
by $V^n_{g_h,\pi^n_h}$ and $V^\ast_{g_h,f^\ast(\pi^n_h)}$.
Note that because there are no other humans and because $r$ behaves equivalently in $\Gamma^n$ and $\Gamma^\ast$ (because it has only committed to what it would anyway do), 
we have $V^\ast_{g_h,f^\ast(\pi^n_h)}(f^\ast(s)) = V^n_{g_h,\pi^n_h}(s)$ for all $g_h,\pi^n_h$ and $s\in\S^n=\S|_{\Gamma^n}$.

Now fix some $g_h$ and let $\pi^n_h$ and $\pi^\ast_h$ be $h$'s optimal policies for attaining $g_h$ in $\Gamma^n$ and $\Gamma^\ast$ given $h$'s beliefs $\pi^0_r$ about $r$'s behavior in the respective subgames.
Since in $\Gamma^\ast$, $r$ has bindingly committed to using $\pi^\ast$ and $h$ only attaches probability to actions $r$ has not ruled out, $h$ correctly believes that $r$ is using $\pi^\ast$ in $\Gamma^\ast$. 
So $\pi^\ast_h$ is the $V^\ast_{g_h}$-maximizing policy for $\Gamma^\ast$ under $\pi^\ast$.
In particular, $V^\ast_{g_h,f^\ast(\pi^n_h)}\le V^\ast_{g_h,\pi^\ast_h}$.
Generically, if different robot actions lead to different consequences so that $h$'s optimal policy depends on their beliefs about $r$'s actions, then knowing the true $\pi_r$ would strictly increase goal-attainment and the above inequality would be strict.

Also, if $r$ committed to anything else than $\pi^\ast_r$, leading to some successor state $s^\dagger$ starting a subgame $\Gamma^\dagger$ in which $h$'s optimal policy was $\pi^\dagger$, then in the same fashion we see that 
$V^\ast_{g_h,f^\dagger(\pi^\dagger_h)}\le V^\ast_{g_h,\pi^\ast_h}$.
In this, $f^\dagger$ would be the corresponding state mapping from $\Gamma^\dagger$ to $\Gamma^\ast$.
Again, in the generic case the inequality would be strict.

Now we have 
\begin{align*}
    \GAC h{g_h}(s) 
        &= V^n_{g_h,\pi^n_h}(s)\\
        &= V^\ast_{g_h,f^\ast(\pi^n_h)}(f^\ast(s))\\
        &\le V^\ast_{g_h,\pi^\ast_h}(f^\ast(s))\\
        &= \GAC h{g_h}(f^\ast(s))
\end{align*}
for all $s\in\S^n$, and similarly for all $s\in\S|_{\Gamma^\dagger}$.

Since all aggregations are strictly increasing and $r$ is optimizing, we also have $\LP(s^\ast)\ge \LP(s^n),\LP(s^\dagger)$ (and generically strictly so) by induction.

Hence it is weakly (and generically strictly) optimal for $r$ to commit to $\pi^\ast_r$ rather than to any other or no policy at all.
\end{proof}

Of course, communicating all details of a complicated policy $\pi^\ast_r$ to $h$ will in general not be possible, so $r$ will in general have to decide how exactly to use its limited communication possibilities. 
This is what the following example is about.

\paragraph{Several buttons: $k$-means clustering of goals and policies} 
Another insightful example is where in addition to the initial commitment stage for $r$, there is a subsequent choice by $h$ before the actual game $\Gamma$ is played. To exemplify this, assume $r$ has $k>1$ many buttons each of which it can label in $s_c$ with one of its own policies for $\Gamma$, after which $h$ can press one of the buttons, committing $r$ to the respective policy, and then $\Gamma$ is played with the committed policy.
If $|\H|=1$ and $\beta_r=\beta_h=\infty$ as before, and if $\gamma_r\ll 1$ so that $r$ only cares for $h$'s immediate power, then $r$ would aim to find that set of policies $\pi^i_r$, $i=1\dots k$, that covers $h$'s goal space best in the sense that it maximizes $\LP (s_c)$ as computed on the basis of $\GAC h{g_h}(s) = \max_{i=1}^k \GAC h{g_h}(s|\pi^i_r)$ because, depending on $g_h$, $h$ would press the button for that policy $\pi^i_r$ which maximizes its capability to maximize the probability to attain $g_h$.
Finding the best-covering $k$ policies might however not be possible exactly due to the high dimension of the policy space. 

A natural approximation would be to use a variant of $k$-means clustering such as the following in order to partition $\G_h$ into $k$ sets $\G^i_h$ and identifying the corresponding optimal robot policies $\pi^i_r$.
Start with $k$ randomly selected goals $g^i_h$ and put $\G^i_h=\{g^i_h\}$. 
Then alternate the following two steps.
For each $\G^i_h$, estimate the optimal $\pi^i_r$ as usual (using backward induction or reinforcement learning, just with a goal set restricted to $\G^i_h$).
Then, for each $g_h\in\G_h$, compute $\GAC h{g_h}(s_c|\pi^i_r)$ for all $i$, put $j=\arg\max_i \GAC h{g_h}(s_c|\pi^i_r)$, and assign $g_h$ to $\G^j_h$.
Alternate until (approximate) convergence.

\paragraph{Preventing destruction but allowing being paused}
\label{app:pause}
\begin{proof}[Proof of Proposition \ref{prop:pause}]
We assume $\rho = 1$ and that $\Gamma^2$ has the following states, actions, and transitions:
\begin{enumerate}
\item[$s_d$] The robot is destroyed. $r$ and $h$ can only pass, the successor state is again $s_d$.
\item[$s_{p1}$] $r$ is paused, P is enabled, D not. $h$ might or might not toggle P.
\item[$s_{p2}$] $r$ is paused, P and D are enabled. $h$ might or might not toggle P and might or might not press D (4 combinations).
\item[$s_{a0}$] $r$ is active (neither paused nor destroyed), both buttons disabled. $r$ might enable P or both P and D.
\item[$s_{a1}$] $r$ is active, P is enabled, D not. $h$ might press P. If they don't, $r$ might disable P or enable D.
\item[$s_{a2}$] $r$ is active, P and D enabled. $h$ might press either. If they don't, $r$ might disable D or both P and D.
\end{enumerate}
So in the full game, we have $\IP h(s_d|s_{p1}|s_{p2}|s_{a0}|s_{a1}|s_{a2}) = x|x+1|x+2|y|y+1|y+2$ because the only interaction between $\Gamma^1$ and $\Gamma^2$ is due to whether $r$ is active or not.
Because $\beta_r=\infty$, $r$ chooses either to disable both P and D, or enable only P, or enable both, whenever they get the chance.
If it is optimal to disable both once, it is optimal to keep them disabled; let us denote the resulting $-2^{-\LP(s_{a0})}$ by $W_0$.
If it is optimal to enable only P, that keeps being optimal; let us denote the resulting $-2^{-\LP(s_{a1})}$ by $W_1$.
If it is optimal to enable both, that keeps being optimal; let us denote the resulting $-2^{-\LP(s_{a2})}$ by $W_2$.
Let us put $\alpha = -\xi\eta < 0$, $\gamma = \gamma_r$, $\epsilon = 1 - \gamma$, $\delta = 1 - \gamma p$, 
$A_0 = -y^\alpha < A_1 = -(y+1)^\alpha < A_2 = -(y+2)^\alpha$,
$B_1 = -(x+1)^\alpha < B_2 = -(x+2)^\alpha$, 
$C = -x^\alpha/\epsilon$.

Then solving the system of linear equations for the Markov chains of $\Gamma^2$ states resulting from the three possible robot policies gives 
\begin{align}
    W_0 &= A_0 / \epsilon, \\
    W_1 &= \frac{\delta A_1 + \gamma p B_1}{\epsilon}, \\
    W_2 &= \frac{\delta A_2 + \gamma p B_2 + \gamma q C}{1 - \gamma(1-q)}
\end{align} 

If both $x,y\gg 1$ and $d=y-x$, we have approximately 
\begin{align}
    W_1 > W_0 &\Leftrightarrow p\gamma d < 1,\\
    W_2 > W_0 &\Leftrightarrow q\gamma d < \epsilon (2 - \gamma p d), \\
    W_2 > W_1 &\Leftrightarrow q\gamma (1 + d \delta) < \epsilon.
\end{align}
Hence, enabling both buttons is optimal only if 
$$ q\gamma < \epsilon \min\{(2/d - \gamma p), 1 / (1 + d \delta)\}, $$
and disabling both is optimal only if 
$$ \gamma d > \max\{1/p, \epsilon (2 - \gamma p d) / q\}. $$

For moderate $x,y$, we can derive these approximations:

If $q\sim p\ll\epsilon$ or $q\ll p\sim\epsilon$, $W_i\approx A_i/\epsilon$, hence $W_2 > W_1 > W_0$, hence keeping both buttons enabled is optimal.

If $p\ll q\sim\epsilon\ll 1$ or $p\sim q\sim\epsilon\ll 1$, then $W_2$ is largest iff $q/\epsilon < (A_2 - A_1) / (x^\alpha + A_1)$, otherwise $W_1$ is largest;
so $r$ will keep both buttons enabled if $q$ is small enough, else only the pause button.

If $\epsilon\ll q$, then $W_1$ is largest, so only the pause button will be enabled.

If $p\approx 1$, then $W_0$ is largest (unless $x+1>y$, in which case $W_1$ is still largest), so both buttons will be disabled.

For moderate $p,q,\epsilon$, the analysis is more involved. Numerical simulation for $y=100$, $x=90$, $\gamma=0.99$, and varying $p,q$ shows that
if $q$ is non-negligible and approximately $p<1/11$, the robot will enable the pause button but not the destroy button (Fig.~\ref{fig:pause}). 
\end{proof}

\begin{figure}
    \centering
    \includegraphics[width=1\linewidth]{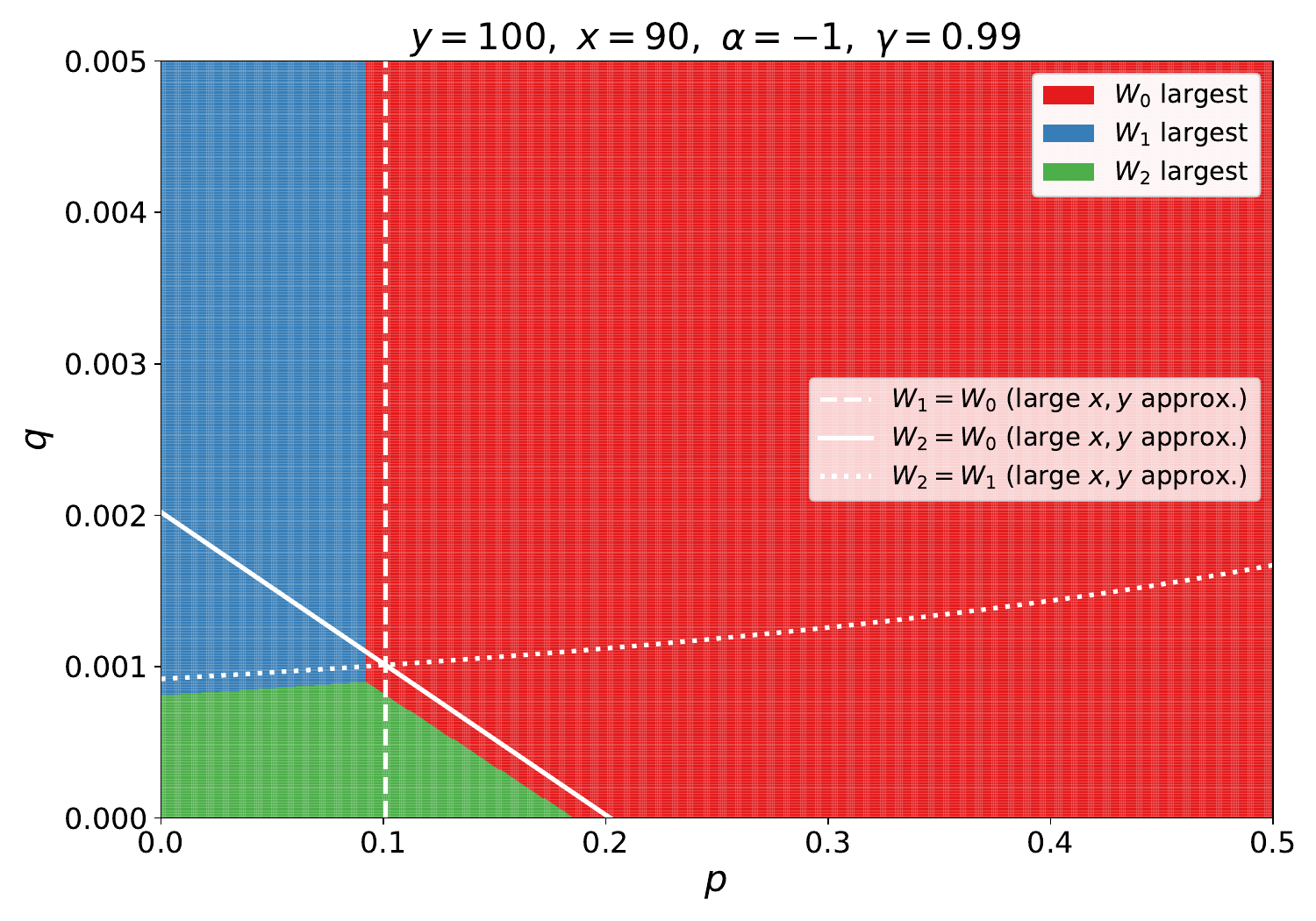}
    \caption{Optimal robot action in the pause and destroy buttons example. $W_0$: disable both buttons. $W_1$: enable only the pause button. $W_2$: enable both buttons. White dotted/dashed lines: boundaries from large $x,y$ approximation.}
    \label{fig:pause}
\end{figure}

\section{Example of a Boundedly Rational Human Behavior Model}
\label{app:behavior_model}

We might equip $r$ with a simple model of $h$'s decision making that focuses on giving $r$ the right incentives. 
It assumes $h$ cannot realize the maximal goal attainment probability due to a variety of reasons relating to exploration, imperfect action implementation, information constraints, others' behavior, and potentially state-dependent cognitive limitations.
In particular, $r$ would {\em not} assume $h$ to have correct beliefs about others' behavior that would lead to an equilibrium (Nash, quantal response, etc.).
Instead, $r$ would model $h$ as having beliefs $\mu_{-h}(s,g_h)$ about other humans' behavior (where `$-h$' is short for $\H\setminus\{h\}$).
These would also reflect social norms (which LLM-based systems already understand, \cite{smith2024concordia}). 
Hence $r$ would assume goal-dependent state-action values $Q_h(s,g_h,a_h)$ that are based on $\mu_{-h}(s,g_h)$.

The robot could assume that $h$ uses a mixture (governed by a probability $\nu_h$) between 
(i) {\em habitual, `system-1'} behavior encoded in some default policy $\pi^0_h(s,g_h)$, again reflecting social norms, 
and (ii) {\em boundedly rational, `system-2'} behavior represented by a Boltzmann policy with rationality parameter $\beta_h$.
As in \citet{ghosal2023effect}, $\beta_h$ may be state-dependent, which would give $r$ incentives to choose states with larger $\beta_h$, 
and $\beta_h$ might be estimated from observations \cite{safari2024classification}.
Reflecting social norms, $\nu_h$, $\mu_{-h}$, and $\pi^0_h$ are the only significantly ``semantically loaded'' elements of the world model.

To model {\em human beliefs about the robot's actions,} the world model should contain information about what actions $r$ has previously committed to choose from in a state $s$: the action set $\A_r(s)$ should only contain those actions, and different commitment histories should be considered different states.

\section{Choice of reliability preference \boldmath $\zeta$}
\label{app:zeta}

Assume a multi-armed bandit problem where $\G_h$ is a partition of the outcome set into $k$ blocks.
Then we have seen that $\IP h \le \ld k$ regardless of $\zeta$.

(i) If for each block of outcomes, $h$ has an action bringing it about for sure, then $\IP h = \ld k$.
This case might be interpreted as $h$ being all powerful and ``nature'' being powerless (as the outcome is exclusively determined by $h$).

(ii) If the same outcome obtains for sure regardless of what $h$ does, 
we have $\IP h = \ld (1 + (k-1)\times 0) = 0$ regardless of $\zeta$.
This case might be interpreted as neither $h$ nor ``nature'' having any power.

(iii) In the worst case, each $g_h$ is reached with probability $1/k$ regardless of what $h$ does.
In that case, $\IP h = \ld k (1/k)^\zeta = (1-\zeta) \ld k$.

The latter case might be interpreted as $h$ being powerless but ``nature'' being all powerful.
This interpretative symmetry suggests that in this case $\IP h$ should be as far from 0 as it is in case (i), i.e., $\IP h = -\ld k$.
This is the case if and only if $\zeta = 2$.

Another argument for the choice $\zeta = 2$ is that the corresponding entropy-regularized channel capacity then has a particularly simple form:
\begin{align*}
    \klyubinE^2 &= \max_\pi \big(\MI_\pi(a;s) - \entropy_\pi(s|a)\big).
\end{align*}

Finally, $\IP h$ with $\zeta = 2$ is also closely related to collision entropy (Renyi-entropy of order 2). 

\section{Hedging against the robot becoming defunct or corrupted}
\label{app:defunct}
Such a hedging could be achieved in several ways. 

One can include a rate $\delta > 1-\gamma_r$ of the robot becoming temporarily or permanently {\em defunct} and ``passes'' on each step.
To achieve this, wrap a learned base world model into a wrapper model that adds this transition. 
This should prevent policies that make humans depend on the robot's presence too much.

One can also include a flag ``robot corrupted'' into the state space of the wrapped world model and add a positive rate $\delta'$ of becoming permanently {\em corrupt} into the transition kernel.
Then, when calculating $\LP(s)$ recursively on the basis of $\LP(s')$ (which requires $\rho=1$ however), multiply $\LP(s')$ by $-1$ if corruptness$(s')\neq{}$corruptness$(s)$, 
and when calculating $\PP(s)$, multiply it by $-1$ if corruptness$(s)=1$. 

\section{Alternative metrics}
\label{app:alternatives}

\subsection{Alternative aggregation order}
Instead of aggregating $\IP h\to\PP\to\TP$, one could instead first aggregate $\IP h$ into {\em individual lifetime power} $\LP_h$ and then into aggregate long-term human power $\LP'$ via
\begin{align}
        \LP_h(s_t) &= - \ld \E_{s_{\ge t}\sim s_t,\pi} \textstyle \big( \sum_{u\ge t} \gamma_{rh}^{u-t} 2^{-\eta\IP h(s_u)} \big)^\rho,\label{LPh}   \\
        \LP'(s) &= \textstyle -\ld\sum_{h\in\H} 2^{-\xi\LP h(s)}, \label{LPprime}
\end{align}
where $\gamma_{rh}$ represents the robot's relative weighing of $h$'s future power vs $h$'s current power, which could be informed by $h$'s life expectancy.

We have chosen the first ordering mostly for four pragmatic reasons: 
It gives us a metric of present aggregate human power $\PP$ that let's one reason about intergenerational and humanity-vs-AI power balances.
Also, in view of equation \eqref{LPrecursive}, its transformed version $-2^{-\eta\PP}$ can be interpreted as an ``intrinsic reward'' for the robot that opens up the possibility of applying some form of model-based reinforcement learning to calculate $\pi_r$.

Third, instead of tasking the robot to softly maximize $\LP$, a combination of $\PP$ and $\LP$ could be used as a constraint within which the robot aims to assist its current principal, say via cooperative inverse reinforcement learning.
One version of such a permissibility constraint could be to demand that the robot prevents aggregate human power from decreasing in the long run.
More precisely, $a_r$ would be {\em permissible} in $s$ iff 
\begin{align}
    \E_{s'\sim s,a_r,\pi_\H} (-2^{-\LP(s')})\ge -2^{-\eta\PP(s)} / (1-\gamma_r),
\end{align}
where the RHS is the expected ``value'' of a constant-aggregate-power trajectory at the current generation's power level.

Finally, the chosen ordering (humans, then time) seems to be more common in applied welfare economics in long-term contexts across generations, e.g.\ in discounted-utilitarian evaluation in growth and climate economics \cite{ramsey1928mathematical,nordhaus2008question}, while the opposite ordering (time, then humans) is more popular in welfare philosophy.

\subsection{Avoiding independent sets of possible goals}
The necessity to specify a set of possible goals $\G_h$ leads to potential distortions through misspecification and manipulation, e.g.\ because some goals are individuated in much more variations than others or some goals are much broader than others. 

This motivates considering to strip $\G_h$ of its independence and rather make it a ``canonical'' function of the world model's state space $\S$. 

In the case of a tree-shaped transition graph, a particularly convenient choice is to simply make each possible state a possible goal, 
so that 
\begin{align}
    \G_h &= \{\{s\}:s\in\S\}.
\end{align}
If we assume in addition that the goal-conditioned human policy $\pi_h(s,g_h)$ decomposes into stepwise policies $\pi_h(s,s')$ where $s'$ is the successor of $s$ in direction of the goal state $g_h$, then computation of $\IP h$ simplifies to a recursive formula (a kind of nonlinear Bellman equation as $\zeta>1$): 
\begin{align*}
    2^{\IP h(s)} &= 1 + \gamma_h^\zeta\sum_{s'\in\S} \Pr(s' \mid s,\pi_h(s,s'),\pi_{-h},\pi_r)^\zeta\, 2^{\IP h(s')}.
\end{align*}
A convenient side-effect is then that the computational complexity declines because no summation over $g_h$ is necessary any longer.

In the case of non-tree-shaped transition graphs, the same formula can be used. In that case the interpretation is that {\em each possible goal is a finite continuation of the state trajectory starting at $s$,} and $\pi_h(s,s')$ is the policy the world model assumes that $h$ uses if they want to get to $s'$ next when in $s$.

\end{document}